\documentclass{article}

\usepackage{microtype}
\usepackage{graphicx}
\usepackage{subfigure}
\usepackage{booktabs} 
\usepackage{fullpage}

\usepackage{hyperref}

\usepackage{algorithm}
\usepackage{algorithmic}

\usepackage{abbrevs}
\usepackage{amssymb}

\usepackage{color}
\usepackage{epstopdf}
\usepackage{algorithm}
\usepackage{nicefrac}

\usepackage[english]{babel}
\usepackage[utf8x]{inputenc}
\usepackage[T1]{fontenc}

\usepackage{amsmath,accents}
\usepackage[colorinlistoftodos, disable]{todonotes}

\usepackage{amsfonts}
\usepackage{amsthm}
\usepackage{multicol}

\usepackage{verbatim}

\usepackage[round,sectionbib]{natbib}

\newtheorem{theorem}{Theorem}[section]
\newtheorem{lemma}[theorem]{Lemma}

\newtheorem{corollary}[theorem]{Corollary}
\newtheorem{definition}{Definition}[section]

\newtheorem{remark}[theorem]{Remark}

\newif\ifcomm
\commtrue

\newif\iflong
\longtrue

\newcounter{assumption}
\renewcommand{\theassumption}{A\arabic{assumption}}

\newenvironment{assumption}[1][]{\begin{trivlist}\item[] \refstepcounter{assumption}%
 {\bf Assumption\ \theassumption\ {\em (#1)} } }{
 \ifvmode\smallskip\fi\end{trivlist}}

\newcommand{\argmin}{\mathop{\rm argmin}}
\newcommand{\argmax}{\mathop{\rm argmax}}

\newcommand{\beq}{\begin{equation}}
\newcommand{\eeq}{\end{equation}}

\ifcomm
   \newcommand\comm[1]{\textcolor{blue}{ #1}}
   \newcommand{\mtodo}[2]{\todo{{\bf #1}: #2}} 
   \def\here#1{{\bf $\langle\langle$#1$\rangle\rangle$}}

\else
   \newcommand\comm[1]{}
   \newcommand{\mtodo}[2]{}
   \def\here#1{}
\fi

\def\be{\begin{equation}}
\def\ee{\end{equation}}

\newname\controllermixture{{\rm \textsc{MDP-policy-mixture}}}

\newname\stableset{{\rm \textsc{independent-set}}}

\newcommand{\X}[0]{\mathcal{X}}

\title{Robustness Guarantees for Mode Estimation with an Application to Bandits}
\author{
Aldo Pacchiano$\thanks{Equal contribution.}$\\
UC Berkeley\\
\and
Heinrich Jiang$\footnotemark[1]$\\
Google Research\\
\and
Michael I. Jordan\\
UC Berkeley\\
}

\begin{document}
\maketitle

\begin{abstract}
Mode estimation is a classical problem in statistics with a wide range of applications in machine learning. Despite this, there is little understanding in its robustness properties under possibly adversarial data contamination. In this paper, we give precise robustness guarantees as well as privacy guarantees under simple randomization. We then introduce a theory for multi-armed bandits where the values are the modes of the reward distributions instead of the mean. We prove regret guarantees for the problems of top arm identification, top m-arms identification, contextual modal bandits, and infinite continuous arms top arm recovery. We show in simulations that our algorithms are robust to perturbation of the arms by adversarial noise sequences, thus rendering modal bandits an attractive choice in situations where the rewards may have outliers or adversarial corruptions.
\end{abstract}

\section{INTRODUCTION}
Work in mode estimation has received much attention (e.g. \citep{parzen1962estimation,chernoff1964estimation,yamato1971sequential,silverman1981using,tsybakov1990recursive,vieu1996note,dasgupta2014optimal}) with practical applications including clustering \citep{cheng1995mean,sheikh2007mode,vedaldi2008quick,jiang2017modal}, control \citep{madani2007backstepping,hofbaur2002mode}, power systems \citep{williams2001mode,sarmadi2013electromechanical}, bioinformatics \citep{hedges2003comparison}, and computer vision \citep{yin2003fast,tao2007color,collins2003mean}; however, to the best of our knowledge, little is known about the statistical robustness of mode estimation procedures despite the popularity of mode estimation and the increasing need for robustness in modern data analysis \citep{dwork2014algorithmic}. Such robustness is important if mode-estimation based learning systems need results to be less sensitive to possibly adversarial data corruption. Moreover, data sources may be more likely to release data to the learner if it can be guaranteed for each source that their additional data will not change the final outcome by much-- in other words robustness is also intimately tied to another important theme of privacy \citep{dwork2009differential}. 

We then provide a new application of mode estimation to the problem of the multi-armed bandits (MAB)
\citep{robbins1985some}, called {\it modal bandits}. MABs have been used extensively in a wide range of practical applications and have been extensively analyzed theoretically. The vast majority of works presume that the value of an arm is the expected value of a reward distribution. In this paper, we present an alternative: where the reward is a function of the modes of the distribution of an arm. This leads to a bandit technique that is more robust and better uses the information from the shape of the arm's distribution as well as other nuances that may be lost with the mean (see Figure~\ref{fig:1}).

Using the mean of the reward distribution can present serious limitations when the observations are biased, potentially due to adversarial interference. We show quantitatively that whenever this is the case, our mode-based bandit algorithms present an alternative to mean-based ones.

Another situation where modal bandits are useful is when the agent already has samples from the arms, but has only one shot to select an arm to pull. Here, the agent may be more interested in optimizing what is ${\it likely}$ to happen rather than the choice that is optimal in expectation. For example, when a risk-averse agent needs to decide between a decision that is likely to have small gains and a decision that has a small chance of high gains but large chance of no effect and prefers the former.

\begin{figure}[h]
\begin{center}

\includegraphics[width=6cm, height=4cm]{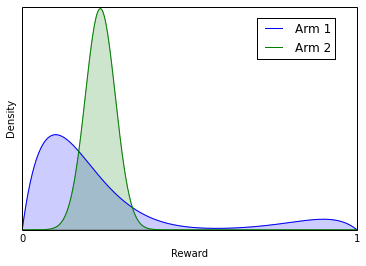}

\end{center}
\label{fig:1}
\caption{\small{Two distributions with the same mean. Their underlying structure can be quite different.}}

\end{figure}

In this paper we assume each arm is a distribution over vectors in $\mathbb{R}^D$ with density $f$ and a set of modes $\mathbb{M} \subset \mathbb{R}^D$. We model the reward of an arm as given by a 'score' function that takes $\mathbb{M}$ as input and outputs a value in $[0,1]$. Although our results can be extended to other more general definitions and more complex modal behaviors, such as scoring functions depending on the value of the $m$-th most likely mode or the distance between the smallest and the largest mode, in this paper we focus on the case when scoring functions depend only on the most likely mode. Details of a more general setting involving scoring functions depending on multiple modes is laid out in the Appendix. We proceed to define the notion of {\it mode} and {\it score function} we will use to analyze the modal bandit problem.

\begin{definition} [Mode]\label{definition:mode}
Suppose that $f$ is a density over $\mathbb{R}^D$. $x$ is a mode of $f$ if $f(x') \le f(x)$ for all $x' \in \mathbb{R}^D$.
\end{definition}

We focus on the case when the score function takes as input the maximum mode. For density $f$ we denote its maximum density mode as $mode(f)$. Since the later is simply a point in $\mathbb{R}^D$:

\begin{definition}[Score Function]
A score function on a density with domain $\mathbb{R}^D$ is a map $r : \mathbb{R}^D \rightarrow [0,1]$. The reward of an arm with associated density $f$ equals $r(\text{mode}(f))$
\end{definition}

We assume that the rewards are in $[0 ,1]$ but it is clear that the results can be extended to any compact interval in $\mathbb{R}$. Definition \ref{definition:mode} can be relaxed to allow modes be local maximas instead of global maximas and our analysis can handle the case where the density has multiple modes and the reward is a function of these modes. We call this relaxed notion $p$-modes and provide analogues of our results for $p$-modes in the Appendix.

\section{CONTRIBUTIONS AND RELATED WORK}

\subsection{Mode Estimation}

\citep{tsybakov1990recursive} gave one of the first nonparametric analyses of mode estimation using a kernel density estimator
for a unimodal distribution on $\mathbb{R}^D$ and established a lower bound estimation rate of $\tilde{\Omega}(n^{1/(4+D)})$. \citep{dasgupta2014optimal} gave an analysis of the $k$-nearest neighbor density estimator and provided a procedure based on this density estimator and nearest neighbor graphs which can recover the modes in a multimodal distribution and attained the minimax optimal rate.  

In Section~\ref{section:pmode}, we present Algorithm~\ref{alg:pmode} which finds the highest density mode. In the Appendix we show this algorithm can be adapted to the case when we may care about the $p$-th highest modes instead. This comes from a simple modification to the mode-seeking algorithms in \citep{dasgupta2014optimal,jiang2017modal}.
We then treat the mode estimation procedure as a black box as it works without any a-priori knowledge of the density and only requires mild regularity assumptions on the density.

We build on mode estimation in the following ways. We show that our mode estimation algorithm is statistically robust to certain amounts of adversarial contamination. We then propose and analyze a differentially private mode estimation algorithm.  To our knowledge, this is the first time robustness or privacy guarantees have been provided for a mode estimation procedure.

In Section~\ref{section:contextual}, we analyze the contextual modal bandit. In order to do this, we must estimate the modes of the arm's {\it conditional} density (conditioned on the context) given samples from the joint density. Thus, we develop a new procedure to estimate the modes of a conditional density given samples from the joint density. We show that it recovers the modes with statistical consistency guarantees and it is practical since it has similar computational complexity to that of Algorithm~\ref{alg:pmode} and again treat it as a black box. Estimating the modes of a conditional density may be of independent interest because a number of nonparametric estimation problems can be formulated in this way \citep{chen2016nonparametric}.

\subsection{Modal Bandits}
We then apply mode estimation results to the stochastic MAB problem where the player chooses an arm index $i \in [K]$ which produces a reward from a density $f_i : [0, 1] \rightarrow \mathbb{R}^D$ with set of modes $\mathbb{M}_i$ and maximum mode $\theta_i$. The player's objective is to choose the density -henceforth referred to as arm - with optimal modal score $r: \mathbb{R}^D \rightarrow \mathbb{R}$. We analyze different problems related to this setup. We start by introducing some results concerning mode estimation in Section \ref{section:pmode}. Our contributions also include analogous results to familiar ones in the classical MAB setting.
\begin{itemize}
\item In Sections~\ref{section:toparm} and~\ref{section:topmarm}, we study top arm identification. We present Algorthm~\ref{alg:ucb}, which is an analogue for the Upper Confidence Bound (UCB) strategy for modal bandits. Theorem~\ref{theo:toparm} then shows that we can recover the top-arm given $n$ pulls where $n$ is in terms of the optimality-gap of the arms. We then present Algorithm~\ref{alg:topM} which is an analogue of UCB to recover the top $m$ arms. Theorem~\ref{theo:topmarm} gives guarantees on recovery of the top $m$ arms.
\item In Section~\ref{section:regret} we introduce two new notions of regret for modal bandits. The first is an analogue of a familiar notion of pseudo-regret from the classical stochastic MAB. The second notion of regret is based on the sample mode estimates, which can be compared to familiar notion of regret computed over sample means. We then attain analogous bounds which are tight up to logarithmic factors. 
\item In Section~\ref{section:contextual}, we introduce contextual modal bandits where the environment samples a context from $\mathbb{R}^d$ from some sampling distribution and is revealed to the learner. We show that a simple uniform sampling strategy can {\it directly} recover the optimal policy {\it uniformly} over the context space. 
\end{itemize}

\subsection{Other approaches to robust bandits}
A recent approach of \citep{szorenyi2015qualitative} uses the quantiles of the reward distributions to value the arm.
This approach indeed combats some of the limitations described above. 
Although using the quantiles of the reward distribution is a simple and reasonable approach in many situations where
using the mean fails, using the modes of the reward distribution has properties which are not offered by 
using the quantiles. 

First, unlike quantiles, our method is robust against constant probability noise so as long as this noise is not too concentrated to form a new mode. Second, if the distribution has rewards concentrated around a few regions, this method {\it adapts} to those regions.
In particular, the learner need not know the locations, shapes, or intensity of these regions--
no {\it a priori} information about the density is needed. If one used the quantiles, then there is still the question of which quantile to choose. 

In the situation where the reward depends on a {\it hidden} and {\it non-stationary} context, the mean and quantile could possibly not even converge while the modes of the reward distribution can remain stable. It is a reasonable assertion that
the performance of an arm can depend on the state of the environment which the learner does not have access to. 
Suppose that the hidden context can take on values $H_1$ or $H_2$
sampled by the environment but not revealed to the learner. If the context is $H_i$, then let the reward be $\mathcal{N}(\mu_i, \sigma^2)$ where $\mathcal{N}$ denotes the normal distribution, $\mu_1 \neq \mu_2$ and $\sigma > 0$. Now suppose that the sampling distribution from which the environment chooses the hidden context is not stationary but can vary over time. In such a situation, both the mean and quantile could change drastically and 
the estimates of mean or quantile can possibly not converge; moreover in this situation any confidence interval typical in MAB analyses is also rendered meaningless 
and thus the learner would fail when using mean or quantiles. However, the modes of the reward distribution ($\mu_1$ and $\mu_2$) will not change.

\section{MODE ESTIMATION}\label{section:pmode}
\subsection{Algorithm and analysis}
In this section, we show how to estimate the mode of a distribution given i.i.d. samples.
The results are primarily adapted from known results about nonparametric mode estimation \citep{dasgupta2014optimal,jiang2017modal}.
The density and mode assumptions are borrowed from \citep{dasgupta2014optimal}. 
\begin{assumption}[Modal Structure]\label{modeassumption1}
A local maxima of $f$ is a connected region $M$ such that the density is constant on $M$ and decays around its boundaries.
Assume that each local maxima of $f$ is a point, which we call a mode.
Let $\mathcal{M}$ be the modes of $f$, which we assume is finite. Then further assume that $f$ is twice differentiable around a neighborhood of each $x \in \mathcal{M}$ and $f$ has a negative-definite Hessian
at each $x \in \mathcal{M}$ and those neighborhoods are pairwise disjoint.
\end{assumption}

\begin{algorithm}[tbh]
   \caption{Estimating the mode}
   \label{alg:pmode}
\begin{algorithmic}
\STATE Input: $k$ and sample points $X = \{X_1,...,X_n\}$.
\STATE Define $r_k(x) := \inf \{ r : |B(x, r) \cap X| \ge k \}$.
\STATE Return $\argmin_{x \in X} r_k(x)$
\end{algorithmic}
\end{algorithm}

\begin{theorem}\label{theo::pmode_final}
Suppose Assumption~\ref{modeassumption1} holds and $f$ is a unimodal density.
There exists $N_f$ depending on $f$ such that for $n \ge N_f$, setting $k = n^{4/(4+D)}$, we have
\begin{align*}
\mathbb{P}\left(|\hat{x} - \text{mode}(f)| \ge  \frac{\sqrt{\log(1/\delta)} \log n}{n^{1/(4+D)}}\right) \le \delta,
\end{align*}
which matches the optimal rate for mode estimation up to log factors for fixed $\delta$. Where $| \cdot |$ denotes the $l_2$ norm of $\mathbb{R}^D$.
\end{theorem}

For the rest of the paper, we will assume these choices and thus Algorithm~\ref{alg:pmode} can be treated as a black-box mode estimation procedure.
Thus, we define the following notion of sample mode:
\begin{definition}
For any set $S$ of i.i.d. samples let $\widehat{\text{mode}}(S)$ be the estimated mode of $S$ from applying Algorithm~\ref{alg:pmode} under the settings of Theorem~\ref{theo::pmode_final}. In particular, the computation of $\widehat{\text{mode}}$ on a set of points is understood to be w.r.t. a confidence setting $\delta$. 
\end{definition}

Let $r:\mathbb{R}^D \rightarrow \mathbb{R}$ be a score function. If $r$ is $1$-Lipschitz, the following corollary holds:

\begin{corollary}
    Assuming the same setup as Theorem \ref{theo::pmode_final}, then:
    \begin{equation*}
        \mathbb{P}\left(| r(\hat{x}) - r(\text{mode}(f))| \geq  \frac{\sqrt{\log(1/\delta)} \log n}{n^{1/(4+D)}  }    \right) \leq \delta
    \end{equation*}
\end{corollary}

Although all of our results hold for densities over $\mathbb{R}^D$, and $L$-Lipschitz score functions $r$, in the spirit of simplicity, in the main paper we mostly discuss the case $D = 1$, score function $r(x) = x$ and density $f$ having domain $[0,1]$. 
\subsection{Robustness of Mode Estimator}
We show that our mode estimation procedure is robust to arbitrary perturbations of the arm's samples. It is already clear that the mode estimates are robust to any perturbation which is sufficiently far away from the mode estimate $\widehat{x}$ and that perturbations don't create high-intensity regions (i.e. there are no samples whose $k$-NN radius is smaller than that of $\widehat{x}$). In such a situation, it is clear that such perturbations will not change the mode estimator.

The result below provides insight into the situation where the perturbation can be chosen adversarially and in particular when such perturbation can be chosen near the original mode estimate. Specifically, we assume there are $\ell$ additional points added to the dataset and the result bounds how much the mode estimate can change. We require $\ell < k$,  because otherwise, an adversary can place the $\ell$ points close together anywhere and create a new mode estimate arbitrarily far away from the original mode estimate when using Algorithm~\ref{alg:pmode}.

\begin{theorem}[Robustness]
\label{theo:stability}
Suppose that $f$ is a unimodal density with compact support $\mathcal{X}\subseteq \mathbb{R}^D$ and $f$ satisfies Assumption~\ref{modeassumption1}. Then there exists constants $C, C_1,C_2,$  depending on $f$ such that the following holds for $n$ sufficiently large depending on $f$.
Let $0 < \delta < 1$ and $\ell > 0$ be the number of samples inserted by an adversary. 
Let $\hat{x}$ be the mode estimate of Algorithm~\ref{alg:pmode} on $n$ i.i.d. samples drawn from $f$ and $\widetilde{x}$ be the mode estimate by Algorithm~\ref{alg:pmode} on that sample along with the $\ell$ inserted adversarial samples. If $k$ satisfies the following,
\begin{align*}
    k &\ge  C_1 \log(1/\delta)^2 \log n + \ell\\
    k &\le C_2 \log(1/\delta)^{2D/(4+D)} (\log n)^{D(4+D)} \cdot n^{4/(4+D)}.
\end{align*}
then with probability at least $1 - \delta$, we have
\begin{align*}
    |\hat{x} - \widetilde{x}| \le C \sqrt{\log(1/\delta)}\cdot (\log n)^{1/4} \cdot (k - \ell)^{-1/4}.
\end{align*}
\end{theorem}

\begin{proof}
Let $x_0$ be the true mode of $f$. 
It suffices to show that for appropriately chosen $\widetilde{r}$, we have
\begin{align*}
    \sup_{x \in B(x_0, r_0)} r_k(x) < \inf_{x \not \in B(x_0, \widetilde{r})} r_{k - \ell} (x),
\end{align*}
where $r_k(x)$ is the $k$-NN radius of any point $x$ and $r_0$ is the distance of $x_0$ to the closest sample drawn from $f$. This is because when inserting $\ell$ points, the adversary can only  decrease the $k$-NN distance of any point up to its $(k - \ell)$-NN distance. Thus, if we can show that the above holds, then it will  imply that  $|\hat{x} - \widetilde{x}| \le \widetilde{r}$.

We have that the above is equivalent to showing the following:
\begin{align*}
     \inf_{x \in B(x_0, r_0)} f_k(x) > \sup_{x \not \in B(x_0, \widetilde{r})} f_{k - \ell} (x)\cdot \frac{k}{k - \ell},
\end{align*}
where $f_k$ is the $k$-NN density estimator. Using $k$-NN density estimation bounds, we have the following for  some constants $C_3, C_4$:
\begin{align*}
    \inf_{x \in B(x_0, r_0)} f_k(x) &\ge f(x_0) -   \frac{C_3\cdot \log(1/\delta) \cdot \sqrt{\log n}}{\sqrt{k}}, \\
    \sup_{x \not \in B(x_0, \widetilde{r})} f_{k - \ell} (x) &\le f(x_0) - C_4\left(\widetilde{r}^2 -  \frac{ \log(1/\delta) \cdot \sqrt{\log n}}{\sqrt{k - \ell}} \right).
\end{align*}
The result then follows by choosing
\begin{align*}
    \widetilde{r}^2 \ge C \frac{ \log(1/\delta) \cdot \sqrt{\log n}}{\sqrt{k - \ell}},
\end{align*}
for appropriate $C$, as desired.
\end{proof}

\subsection{Differentially-Private Mode Estimation}

In some applications such as healthcare, anonymization of the procedure is necessary and there has been much interest in ensuring such privacy \citep{dwork2006our}. As it stands, Algorithm~\ref{alg:pmode} does not satisfy anonymization since the output is one of the input datapoints. We use the $(\epsilon,\delta)$-differential privacy notion of \citep{dwork2006our} (defined below) and show that a simple modification of our procedure can ensure this notion of privacy.
\begin{definition}[Differential Privacy]
A randomized mechanism $\mathcal{M} : \mathcal{D} \rightarrow \mathcal{R}$ satisfies $(\epsilon,\delta)$-differential privacy if any two adjacent inputs $d, d' \in \mathcal{D}$ (i.e. $d$ and $d'$ are sets which differ by at most one datapoint) if the following holds for all $S \subset \mathcal{R}$:
\begin{align*}
    \mathbb{P}\left(\mathcal{M}(d) \in S\right) \le e^{\epsilon}  \mathbb{P}\left(\mathcal{M}(d') \in S\right)  + \delta.
\end{align*}
\end{definition}

\vspace{-.3cm}
\begin{algorithm}[tbh]
   \caption{Differentially Private Mode Estimation}
   \label{alg:dp_mode}
   \begin{algorithmic}
\STATE Input: $k$, $\sigma$, and sample points $X = \{X_1,...,X_n\}$.
\STATE $\widehat{x} := \argmin_{x \in X} r_k(x)$
\STATE Return $\widehat{x} + \mathcal{N}(0, \sigma^2I)$
\end{algorithmic}
\end{algorithm}

To ensure differential privacy, we utilize the Gaussian noise mechanism (see \citep{dwork2006our}) to the final mode estimate. We now show that this method (Algorithm~\ref{alg:dp_mode}) has differential privacy guarantees.

\begin{theorem}\label{theo:privacy}
Suppose that $f$ is a unimodal density with compact support $\mathcal{X}\subseteq \mathbb{R}^D$ and $f$ satisfies Assumption~\ref{modeassumption1}. Then there exists constants $C, C_1,C_2,$  depending on $f$ such that the following holds for $n$ sufficiently large depending on $f$.
Let $0 < \delta < 1$ and $\epsilon > 0$. Suppose that \begin{align*}
\sigma \ge C \log(2/\delta) \cdot (\log n)^{1/4} \cdot k^{1/4 }\cdot \epsilon^{-1}.
\end{align*}.
If $k$ satisfies the following,
\begin{align*}
    k &\ge  C_1 \log(1/\delta)^2 \log n + \ell\\
    k &\le C_2 \log(1/\delta)^{2D/(4+D)} (\log n)^{D(4+D)} \cdot n^{4/(4+D)}.
\end{align*}
then with probability at least $1 - \delta$, Algorithm~\ref{alg:dp_mode} is $(\epsilon,\delta)$-differentially private.
\end{theorem}
\begin{remark}
In particular, we see that taking $k = n^{4/(4+D)}$, we get that $\sigma \approx \log(1/\delta) n^{-1/(4+D)} \cdot \epsilon^{-\epsilon} \rightarrow 0$ as $n\rightarrow 0$.
\end{remark}

\begin{proof}
The result follows by Theorem 1 of \citep{okada2015differentially} and the global sensitivity of estimating the mode from Theorem~\ref{theo:stability}.
\end{proof}

\begin{remark}
For the remainder of the paper, unless noted otherwise, we assume that we use the mode estimator of Algorithm~\ref{alg:pmode} as a black-box using the settings of Theorem~\ref{theo::pmode_final}. It is straightforward to substitute the mode estimation procedure by modify the hyperparameter settings or use a different procedure Algorithm~\ref{alg:dp_mode} appropriately adjusting the guarantees. 
\end{remark}

\section{TOP ARM IDENTIFICATION} \label{section:toparm}
As common to works in best-arm identification e.g. \citep{audibert2010best,jamieson2014best}, we characterize the difficulty of the problem based on the gaps between
the value of the arms to that of the optimal arm and the sample complexity can be written in terms of these. 
\begin{definition}
Let $f_i$ denote the density of the $i$-th arm's reward distribution.
Let $\theta_i$ be the top mode of $f_i$ where $\theta_1 \ge \theta_2 \ge \cdots \ge \theta_K$.
Then we can define the gap between an arm's mode and that of the optimal arm. 
\begin{align*}
\Delta_{i} := \theta_1 - \theta_{i}.
\end{align*}
\end{definition}

Although we've indexed the arms this way, it is clear that the algorithms in this paper are invariant to permutations of arms. 

\begin{algorithm}[tbh]
   \caption{UCB Strategy}
   \label{alg:ucb}
\begin{algorithmic}
	\STATE Input: Total time $n$ and confidence parameter $\delta$.
	\STATE Define $S_i(t)$ be the rewards observed from arm $i$ up to and include time $t$.
	\STATE Let $T_i(t)$ be the number of times arm $i$ was pulled up to and including time $t$. i.e. $|S_i(t)| = T_i(t)$.
	\STATE
	\STATE For $t = 1,...,n$, pull arm $I_t$, where $I_t$ is the following.
	{ \small
	\begin{align*}
	 \argmax_{i=1,...,K} \left\{ \widehat{\text{mode}}(S_i(t-1)) + \frac{\log(1/\delta) \cdot \log(T_i(t-1))}{(T_i(t-1))^{1/(4 + D)}} \right\}.
	\end{align*}
	}
\end{algorithmic}
\end{algorithm}

We give the Upper Confidence Bound (UCB) strategy (Algorithm~\ref{alg:ucb}).
For each arm, we maintain a running estimate of the mode
as well as a confidence band. Then at each round, we pull the arm with the highest upper confidence bound. 
When compared to the classical UCB strategy, we replace the running estimates of the mean and confidence band of the mean with the mode and the confidence band of the mode. Our sample complexities now depend on the confidence bands for mode estimation, which converge at a different rate than that of the mean.

We can then give the following result about Algorithm~\ref{alg:ucb}'s ability to determine the best arm.
\begin{theorem} \label{theo:toparm} [Top arm identification] Suppose $\theta_1 > \theta_2$. Then there exists universal constants $C_0, C_1 > 0$ such that
Algorithm~\ref{alg:ucb} with $n$ timesteps and confidence parameter $\delta / n$ satisfies the following.
If
\begin{align*}
    n \ge \text{PolyLog}\left(\frac{1}{\delta}, \sum_{i=2}^K \Delta_i^{-(4+D)}\right) \cdot  \sum_{i=2}^K \Delta_i^{-(4+D)},
\end{align*}
where $\text{PolyLog}$ denotes some polynomial of the logarithms of its arguments,

then
\begin{align*}
\mathbb{P} \left(\argmax_{i=1,...,K} \big|\{ t : I_t = i, 1 \le t \le n \}\big| = 1 \right) \ge 1 - \delta,
\end{align*}
where $N_{f_i}$'s are constants depending on $f_i$ established in Theorem~\ref{theo::pmode_final}.
\end{theorem}

\begin{remark}
We can compare this to the analogous result for classical MAB \citep{audibert2010best} whose sample complexity (ignoring logarithmic factors) is of order 
$\sum_{i=2}^K \Delta_i^{-2}$ (where the gaps here are w.r.t. the distributional means). Our sample complexity is quintic rather than quadratic in the inverse gaps due to the difficulty of recovering modes compared to recovering means. In fact, for $K = 2$, there exists two distributions such that we require sample complexity at least $\Omega(\Delta_2^{-(4+D)})$ to differentiate between the two distributions. This follows immediately from lower bounds in mode estimation as analyzed in \citep{tsybakov1990recursive}. Thus, our results are tight up to log factors.
\end{remark}

We next introduce a simple uniform sampling strategy and give a PAC bound to obtain an $\epsilon$-optimal arm (which means its mode is within $\epsilon$ of mode of the optimal arm) .
\begin{algorithm}[tbh]
 \caption{Uniform Sampling Strategy}
   \label{alg:uniform}
\begin{algorithmic}
	\STATE Input: Total time $n$ and confidence parameter $\delta$.
	\STATE
	\FOR{$t = 1$ to $n$}
	\STATE Pull arm (where ties are broken arbitrarily)
	\begin{align*}
	I_t := \argmin_{i=1,...,K} \left\{ T_i(t-1) \right\}.
	\end{align*} 
	\ENDFOR
	\STATE $\widehat{\theta}_i := \widehat{\text{p-mode}}(S_i(n))$ for $i = 1,...,K$
	\STATE Return top $k$ arms according to $\widehat{\theta}_i$ value.
\end{algorithmic}
\end{algorithm}

This result can be compared to \citep{even2002pac} for the classical MAB.
\begin{theorem} \label{theo:eoptimalpac} [$\epsilon$-optimal arm identification]
Let $\epsilon > 0$. If we run Algorithm~\ref{alg:uniform} with $n$ at least
\begin{align*}
\max \left\{ K (\log(K) + \log(1/\delta))^5 \epsilon^{-5} \log (\epsilon^{-5}), K \max_{i \in [K]} N_{f_i}\right\},
\end{align*}
then the arm with the highest sample mode is $\epsilon$-optimal with probability at least $1 - \delta$.
\end{theorem}

\begin{proof}
It suffices to choose $n$ large enough such that
\begin{align*}
|\widehat{\text{mode}}(S_i(n)) - \theta_i| \le \epsilon/2.
\end{align*}
Indeed, if this were the case, then if arm $i \neq 1$ was selected as the top arm but not $\epsilon$-optimal, then 
\begin{align*}
\theta_i < \theta_1 - \epsilon &\Rightarrow \theta_i + \epsilon/2 < \theta_1 - \epsilon/2 \\
&\Rightarrow 
\widehat{\text{mode}}(S_i(n))  < \widehat{\text{mode}}(S_1(n)),
\end{align*}
a contradiction. Now from Theorem~\ref{theo::pmode_final} with confidence parameter $\delta/K$, it follows that it suffices to take
\begin{align*}
n \ge K (\log(K) + \log(1/\delta))^5 \epsilon^{-5} \log (\epsilon^{-5}),
\end{align*}
as desired.
\end{proof}

\section{TOP-M ARM IDENTIFICATION}\label{section:topmarm}
We next introduce a strategy to recover the top $m$ arms. Let
\begin{equation*}
\widetilde{\Delta}_i = \begin{cases}
			\theta_i - \theta_{m+1} & \forall \theta_i \geq \theta_{m} \\
			\theta_m - \theta_i, & \forall \theta_i \leq \theta_{m+1}
		\end{cases}
\end{equation*}

\begin{definition}[Confidence bound]
Define 
\begin{align*}
U(t, \delta) = (  \log(ct^2) + \log(1/\delta))\cdot \frac{\log(t)}{t^{1/(4+D)}},
\end{align*}
where $c = \sum_{i=N_0}^\infty \frac{1}{i^2}$.
\end{definition}

The algorithm starts by sampling each arm $N_{f_i}$ number of times. Let $\hat{\theta}_{i, T_i(t)}$ be the empirical mode of arm $i$ at time $t$. In each iteration the algorithm computes confidence bounds of radius $U(T_i(t), \delta/2(K-m))$ for each arm in $H_t$ the set of the $m$ arms with the highest empirical modes at time $t$. For the arms in $L_t = [K]\backslash H_t$ a confidence radius of $U(T_i(t), \delta/2m)$ is used. For all arms in $L_t$ we compute an upper confidence bound $\hat{\theta}_{i,T_i(t)} + U(T_i(t), \delta/2m)$ and for all arms in $H_t$ we compute a lower confidence bound $\hat{\theta}_{i,T_i(t)} - U(T_i(t), \delta/2(K-m))$. The algorithm terminates if either the number of rounds is over or if $h_t$ the lowest lower confidence bound of $H_t$ is larger than $l_t$ the largest upper confidence bound of $L_t$. In case neither termination condition is satisfied, sample $h_t$ with probability $\frac{T_{l_t}(t)}{T_{l_t}(t) + T_{h_t}(t)}$ or $l_t$ otherwise.

\begin{algorithm}[tbh]
   \caption{Top $m-$arms UCB strategy}
   \label{alg:topM}
\begin{algorithmic}
	\STATE Input: Confidence parameter $\delta$.
	\STATE Pull $N_0 = \max_{i}N_{f_i}$ times each arm $i \in [K]$. 
	\FOR{$t = N_0\cdot K \cdots $}
	\STATE $H_t:=$ arms with the highest empirical mode.
	\STATE $L_t:= [ K] \backslash H_t$.
	\STATE $h_t:= \argmin_{i\in H_t} \hat{\theta}_{i,T_i(t)} - U(T_i(t), \delta/ (2(K-m)))$ 
	\STATE $l_t:= \argmax_{i\in L_t} \hat{\theta}_{i, T_i(t)}  + U(T_i(t), \delta/(2m))$
	\STATE $b_{l_t} := \hat{\theta}_{l_t, T_{l_t}(t)} + U(T_{l_t}(t),\delta/((2m)) )$
	\STATE $b_{h_t} := \hat{\theta}_{h_t,T_{h_t}(t)} - U(T_{h_t}, \delta/(2(K-m))) $
	\IF{$ b_{h_t} \geq b_{l_t}$}
	\STATE \quad \text{Return }$H_t$
	\ELSE
	\STATE With probability $T_{l_t}(t)/ ( T_{h_t}(t) + T_{l_t}(t))$ sample $h_t$; otherwise, sample $l_t$.
	\ENDIF
	\ENDFOR
	\STATE \quad \text{Return } $H_t$.
\end{algorithmic}
\end{algorithm}

We then provide corresponding high-probability guarantees on recovering the correct set of arms, given sufficient arm pulls.
\begin{theorem} \label{theo:topmarm}[Top $m$ arm identification]
There is a universal constant $C_o$ for such that with probability at least $1 - \delta$ Algorithm \ref{alg:topM} outputs the correct set of arms provided the number of arm pulls $n$ satisfies:

\begin{align*}
    n \ge \text{PolyLog}\left(\frac{1}{\delta}, \sum_{i=1}^K \widetilde{\Delta}_i^{-(4+D)}\right) \cdot \sum_{i=1}^K \widetilde{\Delta}_i^{-(4+D)}.
\end{align*}
\end{theorem}

\begin{remark}
Algorithm~\ref{alg:topM} and Theorem~\ref{theo:topmarm}  gives us a sample complexity of $\widetilde{O}(\sum_i \widetilde{\Delta}_i^{-(4+D)})$.
This can be compared to top-M arm UCB results in classical multi-armed bandits where the complexity is $\widetilde{O}(\sum_i \widetilde{\Delta}_i^{-2})$.
See e.g. \citep{gabillon2012best, jiang2017practical} for such results. 
\end{remark}
\vspace{-.6cm}

\section{REGRET ANALYSIS}\label{section:regret}
We introduce the following notions of regret based on the modes. 
\begin{align*}
\mathcal{R}(n) &= n \cdot \max_{i=1,...,K} \theta_i - \sum_{j=1}^n \theta_{I_j}, \\
\overline{\mathcal{R}}(n) &= \max_{i=1,...,K} n \cdot \widehat{\text{mode}}\left( \{X_{i, t} : 1 \le t \le n\} \right) \\
&- \sum_{i=1}^K T_i(n) \cdot \widehat{\text{mode}} (\{X_{i, t} : I_t = i, 1 \le t \le n \}).
\end{align*}
The regret thus rewards the strategy with the mode ($\mathcal{R}_n$) or the sample mode ($\overline{\mathcal{R}}(n)$)
 of all trials for a particular arm rather than the mean as in classical formulations.
 
We next give a regret bounds for Algorithm~\ref{alg:ucb}. For $\mathcal{R}(n)$, we attain a poly-logarithmic regret in the number of time steps, while for $\overline{\mathcal{R}}(n)$ we attain a regret of order $\widetilde{O}(n^{4/(4+D)})$. The extra error from the latter is incurred from the errors in the mode estimates.
\begin{theorem} \label{theo:ucbregret}
Suppose $\theta_1 > \theta_2$.
Then with probability at least $1 - \delta$, the regret of Algorithm~\ref{alg:ucb} with $n$ time steps and confidence parameter $\delta/n$ satisfies
\begin{align*}
\mathcal{R}(n) &\le  \text{PolyLog}\left(\frac{1}{\delta}, n \right) \cdot \sum_{i=2}^K \Delta_i^{-(3+D)}  \\
\overline{\mathcal{R}}(n) &\le \mathcal{R}(n) + O\left(\left(\text{PolyLog}\left(\frac{1}{\delta}, n \right) + K\right) \cdot n^{\frac{3+D}{4+D}} \right).
\end{align*}
\end{theorem}

\begin{remark}
We can compare this result for $\mathcal{R}(n)$ to that of the classical notion of pseudo-regret, defined below, which also achieves logarithmic regret.
\begin{align*}
n \max_{i=1,...,K} \mu_i - \sum_{j=1}^n E[\mu_{I_j}],
\end{align*}
where $\mu_i$ is the mean of the $i$-th arm's reward distribution.
\end{remark}

\section{CONTEXTUAL BANDITS} \label{section:contextual}
Here, at each time step $t$, the environment samples a context $X_t \in \mathbb{R}^D$ from a context density $f_X$ with compact support $\mathcal{X} \subseteq \mathbb{R}^D$. 
The agent will have knowledge of $X_t$ when making a decision.
Then for each arm $i$, the reward $r \in [0, 1]$, along with the context $X$ has joint density $f_i(r, X)$. We note that it is straightforward to extend the analysis to multi-dimensional arm outputs. We assume $1$-dimensional output to simplify the technical analysis.

Then the optimal policy (context-dependent) in the mode sense is the following
\begin{align*}
\pi(x) := \argmax_{i=1,...,K} \left\{  \text{mode}(f_i(r | X = x)) \right\},
\end{align*}
where the conditional density can be written as $f_i( r | X=x) = f_i(r, x) /f_X(x)$ by Bayes rule. Then the modes of the conditional distribution are the modes of $f_i(r, x)$ constrained to a fixed $x$. Algorithm~\ref{alg:pmode_conditional} extends known mode estimation results to solve for this.

\begin{algorithm}[H]
   \caption{Estimating modes of conditional densities}
   \label{alg:pmode_conditional}
\begin{algorithmic}
	\STATE Input: $m$, $k$, samples $S_{[n]} := \{(r_1, X_1),...,(r_n,X_n)\}$, $x \in \mathbb{R}^D$, confidence parameter $\delta$, and $\epsilon > 0$.
	\STATE Let $r_k(r,x)$ be the distance from $(r, x) \in \mathbb{R} \times \mathbb{R}^{D}$ to its $k$-th nearest neighbor in $S_{[n]}$. 
   \STATE Let $R := \{ \frac{1}{m}, \frac{2}{m}, ..., \frac{m-1}{m} \}$.
   \STATE \textbf{return} $\text{argmin}_{R_i \in R} r_k(R_i, x)$. 
\end{algorithmic}
\end{algorithm}

For the analysis of Algorithm~\ref{alg:pmode_conditional}, we make a few additional regularity assumptions (formally as Assumption~\ref{assumption:conditional} in the Appendix due to space). This assumption ensures the reward densities have smoothness jointly over the reward and context and that the modes of the conditional density satisfy similar assumptions as Assumption~\ref{modeassumption1}. We now give a sketch of the constrained mode estimation result, the formal version is Theorem~\ref{constrained_pmode} in the Appendix:

\begin{theorem}[Estimating constrained modes]
Let $f$ satisfy Assumption~\ref{assumption:conditional} and the density conditional for each context is $\alpha$-H\"older continuous.
Suppose we have $n$ i.i.d. samples $S_{[n]}$ from $f$. There exists constant $C_0 > 0$ depending on $f$ such that the following holds for $n$  sufficiently large depending on $f$,
\begin{align*}
\mathbb{P}\Bigg(\sup_{x \in \mathcal{X}} &|\widehat{\text{mode}}(S_{[n]}, x) - \text{mode}(f(r|X = x))| \\
&\ge C_0 \frac{\sqrt{\log(1/\delta)} \log n}{n^{\alpha/(4\alpha+2D+2)}}\Bigg) \le \delta.
\end{align*}
where $\widehat{\text{mode}}(S_{[n]}, x)$ is the output of Algorithm~\ref{alg:pmode_conditional} for context $x \in \mathcal{X}$.
\end{theorem}

\begin{algorithm}[H]
   \caption{Uniform Sampling Strategy for Contextual Bandits}
   \label{alg:uniformcontextual}
\begin{algorithmic}
	\STATE Input: Total time $n$ and confidence parameter $\delta$.
	\STATE
	\FOR{$t = 1$ to $n$}
	\STATE Pull arm (where ties are broken arbitrarily)
	\begin{align*}
	I_t := \argmin_{i=1,...,K} \left\{ T_i(t-1) \right\}.
	\end{align*} 
	\ENDFOR
	\STATE \quad $\widehat{\pi}(x) := \argmax_{i=1,...,K} \widehat{\text{mode}}(S_i(n), x)$.
\end{algorithmic}
\end{algorithm}

We now give the result for Algorithm~\ref{alg:uniformcontextual}. We show that the algorithm can learn the optimal policy {\it simultaneously} for arbitrary context $x \in \mathcal{X}$.
\begin{theorem} \label{theo:contextual}
Let $\Delta, \delta > 0$.
If
\begin{align*}
n \ge  \widetilde{\Omega}\left(\Delta^{-(4\alpha + 2D + 2)/\alpha} \right)
\end{align*}
then Algorithm~\ref{alg:uniformcontextual} with confidence parameter $\delta/K$, using the conditional mode estimation procedure of Algorithm~\ref{alg:pmode_conditional} with the settings of Thereom~\ref{constrained_pmode} satsifies
\begin{align*}
\mathbb{P}\Bigg(\widehat{\pi}(x) \text{ is $\Delta$-optimal} \hspace{0.2cm} \forall x \Bigg) > 1 - \delta.
\end{align*}
In particular, as $\Delta \rightarrow 0$, $\widehat{\pi}$ gives a $\Delta$-optimal policy {\it uniformly} over $\mathcal{X}$.
\end{theorem}

This result shows that a uniform sampling strategy can give us guarantees everywhere in context space simultaneously.

\section{EXPERIMENTS}

\textbf{Robustness}. In Figure~\ref{fig:simulations1}, we test the robustness of Algorithm \ref{alg:ucb} to perturbations of the arms. We consider the case when the score function equals the identity. The red (Arm 1) density's mode has the largest reward value. With probability $0.2$ we receive a sample from a noise sequence denoted by the marked points on the $x$-axis. The colors of these points correspond to which arm we perturb.

\begin{figure}[H]
\begin{center}
\includegraphics[width=5cm, height=3cm]{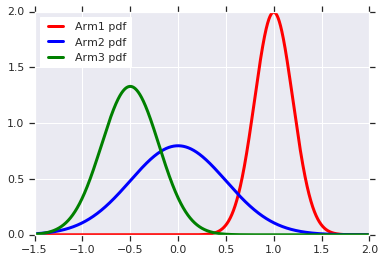}
\includegraphics[width=5cm, height=3cm]{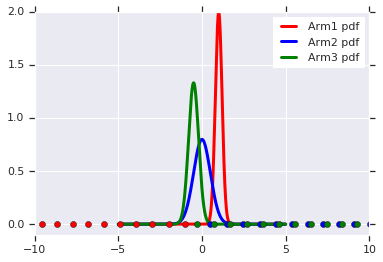}
\includegraphics[width=5cm, height=3cm]{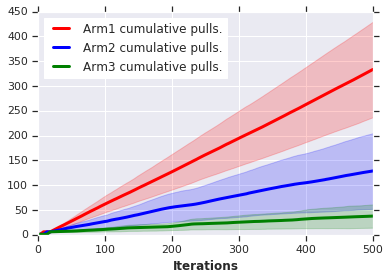}
\caption{{\bf Left}: Three arms.  {\bf Center}: Perturbations.  {\bf Right}: Algorithm \ref{alg:ucb} cumulative arm pulls.\label{fig:simulations1}}
\end{center}
\end{figure}

Based on the reward distribution given in Figure 2, Algorithm \ref{alg:ucb} pulls Arm $1$ (the arm with the highest modal score) more often despite the perturbations experienced by this arm being negative and the perturbations of the remaining arms being positive values. We average over 25 random seeds and mark the standard deviation. 

\textbf{Fine-grained Sensitivity.} In Figure~\ref{fig:simulations2}, we show Algorithm 3 distinguishes between arms with very close modes. We again consider the identity score function. The red (Arm 1) density's mode has the largest modal reward. We average over 25 random seeds and mark the standard deviation.
\begin{figure}[H]
\begin{center}
\includegraphics[width=5cm, height=3cm]{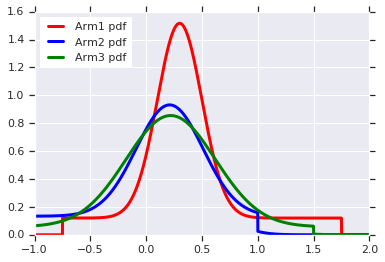}
\includegraphics[width=5cm, height=3cm]{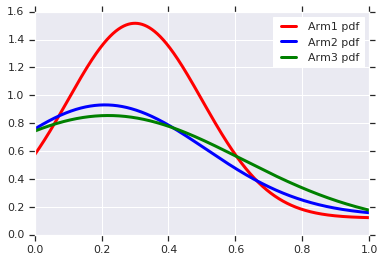}
\includegraphics[width=5cm, height=3cm]{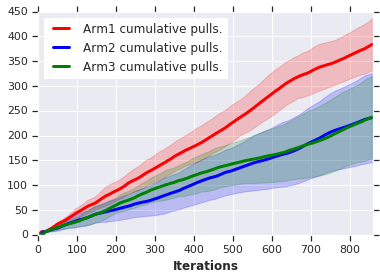}
\caption{{\bf Left}: Three arms. {\bf Center}: Zoomed in view. {\bf Right}: Algorithm \ref{alg:ucb}'s cumulative arm pulls. \label{fig:simulations2}}
\end{center}

\end{figure}

\textbf{Finding arms with furthest mode}. In Figure~\ref{fig:simulation3}, we show Algorithm \ref{alg:ucb} works with score functions other than the identity. In this case we demonstrate it is able to find the arm whose highest density mode is furthest away from the origin-- that is the score function equals the distance of the arm's most likely mode to the origin. In this setup Arm 3 is optimal. We also plot the Regret and Normalized Regret (we divide the regret by the iteration index) using the distance from the origin to the arm's mode as reward. 
\begin{figure}[H]
\begin{center}
\includegraphics[width=0.23\textwidth]{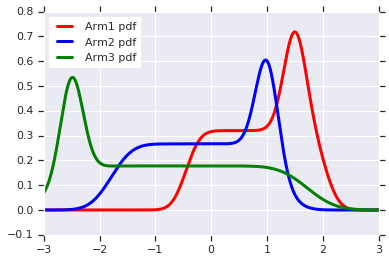}
\includegraphics[width=0.23\textwidth]{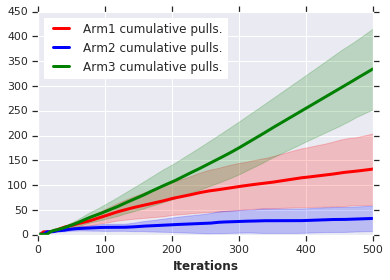}
\includegraphics[width=0.23\textwidth]{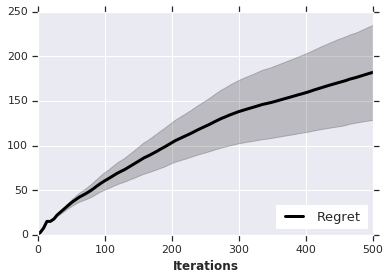}
\includegraphics[width=0.23\textwidth]{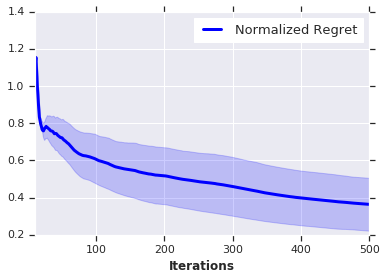}
\end{center}
\caption{\label{fig:simulation3} {\bf Upper Left}: Three arm densities. 
{\bf Upper Right}: Arm pulls. 
{\bf Lower left}:  Regret.
{\bf Lower right}: Normalized regret. } 
\end{figure}
 The plot shows Algorithm \ref{alg:ucb} learns to choose the arm showing outlier behavior as time goes on and does so in a way minimizing the regret captured by differences in outlier score. We average over 25 random seeds and mark the standard deviation. Multimodal experiments and their supporting theory are shown in the Appendix.

\section{CONCLUSION}

In this paper, we've provided two contributions which are of independent interest: (i) robustness and privacy guarantees for mode estimation and (ii) a new application of mode estimation the bandit problem, which we call {\it modal bandits}. To our knowledge, we give the first robustness and privacy guarantees for mode estimation, a popular practical method with a long history of theoretical analysis. We then give an extensive analysis of the modal bandits problems, including best-arm identification, regret bounds, contextual modal bandits, and infinite armed bandits. We include simulations showing that modal bandits indeed can provide robustness to adversarial corruption, thus suggesting that modal bandits can be an attractive choice in settings where robustness is important.

\bibliographystyle{plainnat}
\bibliography{arxivbib}

\appendix

\onecolumn

\section{Definitions}
We start by considering a generalization of the modal objective which we call $p$-mode.

\begin{definition} [$p$-mode]
Suppose that $f$ is a density over $[0, 1]$. $x$ is a mode of $f$ if $f(x') < f(x)$ for all $x' \in B(x, \delta) \backslash \{x \}$ for some $\delta > 0$ where $B(x, \delta) := \{x' \in \mathcal{X} : |x - x'| \le \delta\}$.
Let $\mathcal{M}(f) = \{j_1, j_2,...j_{\ell} \}$ be the modes of $f$ 
with $f(j_1) \ge \cdots \ge f(j_p) > f(j_{p+1}) \ge \cdots \ge f(j_\ell)$.
Then the $p$-mode of $f$ is defined as 
\begin{align*}
\text{p-mode}(f) := \min \{j_1,...,j_{\min\{p, \ell \}} \}.
\end{align*}
\end{definition}

An additional assumption we need for the $p$-mode analysis:
\begin{assumption}\label{modeassumption2}
$f$ is $\alpha$-H\"older continuous. i.e. $|f(x) - f(x')| \le C_\alpha \|x - x'\|^\alpha$ for some $C_\alpha > 0$ and $0 < \alpha \le 1$.
\end{assumption}

\begin{figure}[h]
\begin{center}

\includegraphics[width=6cm, height=4cm]{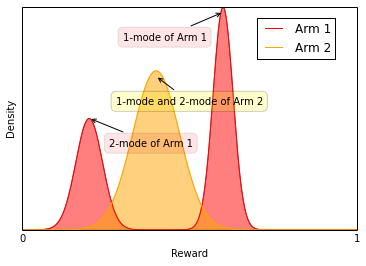}
\end{center}

\label{fig:1_1}
\caption{ Examples of $p$-modes in two reward distributions plotted together. }
\end{figure}

\textbf{ When $p = 1$, the $p$-mode is simply the mode.}

\textbf{All of the subsequent proofs are stated using the terminology of $p-$modes. As stated before setting $p=1$ in each of the subsequent theorem statement yields the proofs of the theorems referenced in the main paper. }

\section{$p$-mode algorithm}

\begin{algorithm}[tbh]
   \caption{Estimating $p$-modes}
   \label{alg:pmode_full}
\begin{algorithmic}
	\STATE Input: $k$ and sample points $X = \{X_1,...,X_n\}$ and confidence parameter $\delta$.
   \STATE Define $r_k(x) := \inf \{ r : |B(x, r) \cap X| \ge k \}$, $f_k(x) := k/(2 \cdot n\cdot r_k(x))$.
   \STATE Let $G(\gamma)$ be the graph with points $X_i \in X$ satisfying $f_k(X_i) \ge \gamma$ and edge between $X_i, X_j$ if $||X_i - X_j|| \le \min \{ r_k(X_i), r_k(X_j) \}$.
   \STATE Initialize $\mathcal{A}:= \emptyset$. Define $\beta_k =  100\cdot \log(1/\delta) \cdot \sqrt{\log n} / \sqrt{k}$. 
   \STATE Sort the $X_i$'s in descending order of $f_k$ values. 
   \FOR{$i=1$ {\bfseries to} $n$}
   \STATE Define $\lambda := f_k(X_i)$ and let $A$ be the connected component of $G(\lambda - \beta_k \lambda)$ that contains $X_i$.
   \STATE {\bf if} $A$ is disjoint from all sets in $\mathcal{A}$ {\bf then}  Add $A$ to $\mathcal{A}$. {\bf end if}
   \STATE {\bf if} $|\mathcal{A}| = p$ {\bf then}  break. {\bf end if}
   \ENDFOR
   \STATE \textbf{return} $\hat{x} := \min \{ \text{argmax}_{x \in A} f_k(x) : A \in \mathcal{A} \}$. 
\end{algorithmic}
\end{algorithm}

\section{Proofs of $p$-mode estimation results}

\begin{lemma} [Lemma 5 of \citep{dasgupta2014optimal}]
Let $f$ satisfy Assumption~\ref{modeassumption1}. There exists $r_M > 0$ sufficiently small and $\check{C}, \hat{C}, \lambda_0 > 0$ such that
the following holds for all $x_0 \in \mathcal{M}$ simultaneously.
\begin{align*}
\check{C} \cdot |x_0 - x|^2 \le f(x_0) - f(x) \le \hat{C} \cdot |x_0 - x|^2
\end{align*}
for all $x \in A_{x_0}$ where $A_{x_0}$ is a connected component of $\{ x : f(x) \ge \lambda_{x_0} \}$ for some $\lambda_{x_0} > \lambda_0$ and 
$A_{x_0}$ contains $B(x_0, r_M)$.
\end{lemma}
The next result guarantees that the modes are separated by sufficiently wide and deep valleys.
\begin{lemma}[Follows from Proposition 1 of \citep{jiang2017modal}]
There exists $r_M > 0$ sufficiently small such that
for each $x_0 \in \mathcal{M}$, there exists a set $S$ such that the following holds for each $x_0' \in \mathcal{M}$ with $x_0' \neq x_0$.
Each path from $x_0$ to $x_0'$ crosses $S$ and
\begin{align*}
\inf_{x \in B(x_0, r_M) \cup B(x_0', r_M)} f(x) > \sup_{x \in S + B(0, r_M)} f(x).
\end{align*}
\end{lemma}

Theorem~\ref{theo::pmode_final} is a corollary of the following.
\begin{theorem} \label{theo:pmode}
Let Assumptions~\ref{modeassumption1},~\ref{modeassumption2} hold.
Suppose that $x_1,...,x_n$ are i.i.d. observations drawn from a density $f$ on $[0, 1]$.
Let $x_p$ be the unique $p$-mode of $f$.
Then there exists $C_1, C_2 > 0$ depending on $\lambda_0$, $r_M$, $||f||_\infty$, $\hat{C}$, $\check{C}$, $C_\alpha$, $\alpha$, and $f(m_p) - f(m_{p+1})$ such that if $k$ satisfies
\begin{align*}
C_1\cdot   \log n \cdot \log(1/\delta) \le k \le C_2  \cdot (\log(1/\delta))^{2/(4+D)} \cdot (\log n)^{1/(4+D)} \cdot n^{4/(4+D)},
\end{align*}
and $M > \sqrt{||f||_\infty / \check{C}}$,
then $\hat{x}$ in Algorithm~\ref{alg:pmode} satisfies the following.
\begin{align*}
\mathbb{P}\left(|\hat{x} - x_p| \ge \frac{M \sqrt{\log(1/\delta)} \cdot (\log n)^{1/4}}{k^{1/4}} \right) \le \delta.
\end{align*}
\end{theorem}

\begin{proof} [Proof of Theorem~\ref{theo:pmode}]
It follows from Theorem 3 and 4 of \citep{jiang2017modal} that for $p = |\mathcal{M}|$ and $C_1, C_2$ appropriately chosen, Algorithm 1 computes $\widehat{\mathcal{M}} := \{ \text{argmax}_{x \in A} f_k(x) : A \in \mathcal{A} \}$ such that
there is a one-to-one mapping between $\mathcal{M}$ and $\widehat{\mathcal{M}}$ which satisfies the following.
If $x_0 \in \mathcal{M}$ corresponds to $\widehat{x_0} \in \widehat{\mathcal{M}}$ then,
\begin{align*}
\mathbb{P}\left(|\widehat{x_0} - x_0| \ge \frac{M\cdot \sqrt{\log(1/\delta)} \cdot (\log n)^{1/4}}{k^{1/4}} \right) \le \frac{1}{n}.
\end{align*}
What remains is showing that for $p < |\mathcal{M}|$, Algorithm~\ref{alg:pmode} stops exactly after choosing the $p$ modes with highest densities.
If $m_1,...,m_\ell$ are the modes where $f(m_1) \ge f(m_2) \ge \cdots \ge f(m_\ell)$, it thus suffices to show that
\begin{align*}
\min_{j=1,...,p} f_k(\widehat{m}_j) > \max_{j=p+1,...,|\mathcal{M}|} f_k(\widehat{m}_j).
\end{align*}
where $\widehat{m}_j$ is the estimate corresponding to $m_j$.
We have
\begin{align*}
\min_{j=1,...,p} f(m_j) - \max_{j=p+1,...,|\mathcal{M}|} f(m_j) = f(m_p) - f(m_{p+1}):= \delta_p > 0.
\end{align*}
By Lemma 2 of \citep{jiang2017modal}, for $C_1$ and $C_2$ chosen appropriately, we have $\widehat{m}_j \in B(m_j, r_M)$; combining this with Theorem 1 of \citep{jiang2017modal}, we have
$\widehat{m}_j = \text{argmax}_{x \in X \cap B(m_j, \tilde{r})} f_k(x)$ where $\tilde{r} = \frac{M \sqrt{\log(1/\delta)} (\log n)^{3/4}}{k^{1/4}}$.
It thus suffices to show that
\begin{align*}
\min_{j=1,...,p} \inf_{x \in B(m_j, \tilde{r})} f_k(x)  > \max_{j=p+1,...,|\mathcal{M}|} \sup_{x \in B(m_j, \tilde{r})} f_k(x).
\end{align*}
We have
\begin{align*}
\inf_{x \in B(m_j, \tilde{r})} f(x) \ge  f(m_j) - \hat{C} \tilde{r}^2.
\end{align*}
By Lemma 4 of \citep{dasgupta2014optimal}, for $C_1$ and $C_2$ chosen appropriately, we have
\begin{align*}
\inf_{x \in B(m_j, \tilde{r})} f_k(x) \ge  \left(1 - \frac{\beta_k}{\sqrt{k}}\right)\cdot (f(m_j) - \hat{C} \tilde{r}^2).
\end{align*}
By Lemma 3 of \citep{dasgupta2014optimal}, for $C_1$ and $C_2$ chosen appropriately, we have
\begin{align*}
\sup_{x \in B(m_j, \tilde{r})} f_k(x) \le \left(1 + \frac{\beta_k}{\sqrt{k}}\right)\cdot f(m_j).
\end{align*}
It is clear that from combining these two, it suffices to choose $k$ sufficiently large (depending on $\delta_p$, $\hat{C}$ and $\tilde{r}$) such that
\begin{align*}
\left(1 - \frac{\beta_k}{\sqrt{k}}\right) \cdot(f(m_p) - \hat{C} \tilde{r}^2) \ge \left(1 + \frac{\beta_k}{\sqrt{k}}\right)\cdot f(m_{p+1}).
\end{align*}
The result follows immediately.
\end{proof}

\section{Proofs of Top-Arm Identification Results}

In this section we prove guarantees for Algorithm \ref{alg:ucb_pmode}, a more general version of Algorithm \ref{alg:ucb}.

\begin{algorithm}[tbh]
   \caption{UCB Strategy}
   \label{alg:ucb_pmode}
\begin{algorithmic}
	\STATE Input: Total time $n$ and confidence parameter $\delta$.
	\STATE Define $S_i(t)$ be the rewards observed from arm $i$ up to and include time $t$.
	\STATE Let $T_i(t)$ be the number of times arm $i$ was pulled up to and including time $t$. i.e. $|S_i(t)| = T_i(t)$.
	\STATE
	\STATE For $t = 1,...,n$, pull arm $I_t$, where $I_t$ is the following.
	{ \small
	\begin{align*}
	 \argmax_{i=1,...,K} \left\{ \widehat{\text{mode}}(S_i(t-1)) + \frac{\log(1/\delta) \cdot \log(T_i(t-1))}{(T_i(t-1))^{1/(4+D)}} \right\}.
	\end{align*}
	}
\end{algorithmic}
\end{algorithm}

\begin{lemma} \label{ucbbound}
Define $N_0 := \max_{i \in [K]} N_{f_i}$.
In Algorithm~\ref{alg:ucb}, with confidence parameter $\delta/n$, the following holds with probability at least $1 - \delta$ simultaneously for $i = 2,...,K$.
\begin{align*}
T_i(n)  \le 2 ((\log n + \log(1/\delta)) \log n)^{4+D} \cdot \Delta_i^{-(4+D)}+ 2 N_0.
\end{align*}
\end{lemma}

\begin{proof}
The proof borrows ideas from the classical analogue. i.e. Theorem 2.1 of \citep{bubeck2012regret}. The main differences are in the
used concentration bounds.
For $T_i(t) > N_0$, we have confidence intervals for each arm index $i$ as follows:
\begin{align*}
\theta_i &\in \left( \widehat{\theta}_{i, t} -  \widehat{\sigma}_{i, t},  \widehat{\theta}_{i, t} + \widehat{\sigma}_{i, t},\right),
\end{align*}
where
\begin{align*}
\widehat{\theta}_{i, t} := \widehat{\text{p-mode}}(S_i(t-1)) \text{ and }
\widehat{\sigma}_{i, t} := (\log n + \log(1/\delta)) \cdot \frac{\log(T_i(t-1))}{(T_i(t-1))^{1/(4+D)}}.
\end{align*}
By a union bound, this holds simultaneously for all $i$ and $t \le n$ with probability at least $1 - \delta$.
If $I_t = i \neq 1$, then at least one of the following holds.
\begin{align*}
\theta_1 \ge \widehat{\theta}_{1, t} + \widehat{\sigma}_{1, t},  \hspace{1cm} 
\widehat{\theta}_{i, t} > \theta_i + \widehat{\sigma}_{i, t}, \hspace{1cm} \Delta_i / 2 < \widehat{\sigma}_{i, t}.
\end{align*}
Otherwise, we have
\begin{align*}
\widehat{\theta}_{1, t} + \widehat{\sigma}_{1, t}  > \theta_1 =  \theta_i + \Delta_i \ge \theta_i + 2\widehat{\sigma}_{i, t} \ge \widehat{\theta}_{i, t} + \widehat{\sigma}_{i, t}.
\end{align*}
Now if $T_i(t-1) > 2 ((\log n+ \log(1/\delta)) \log n)^5 \cdot \Delta_i^{-5}$, then $\widehat{\sigma}_{i, t} < \Delta_i / 2$.
Thus, we have
\begin{align*}
T_i(n) &\le 2 ((\log n + \log(1/\delta)) \log n)^{4+D} \cdot \Delta_i^{-(4+D)}+ \sum_{t = 1}^n  1 \{ \widehat{\theta}_{i, t} > \theta_i + \widehat{\sigma}_{i, t} \} + \sum_{i = 1}^n  1 \{ \widehat{\theta}_{i, t} > \theta_i + \widehat{\sigma}_{i, t} \} \\
&\le 2 ((\log n + \log(1/\delta)) \log n)^{4+D} \cdot \Delta_i^{-(4+D)}+ 2 N_0.
\end{align*}
As desired.
\end{proof}

\begin{proof} [Proof of Theorem~\ref{theo:toparm}]
We have by Lemma~\ref{ucbbound} that
\begin{align*}
T_1(n) = n - \sum_{i=2}^K T_i(n) \ge n - 2^{4+D}(\log^{(8+2D)} n + \log^{(4+D)}(1/\delta) \log^{4+D} n) \sum_{i=2}^K \Delta_i^{-(4+D)}.
\end{align*}
Thus it suffices to have $n \ge 2^{4+D}(\log^{8+2D} n + \log^{4+D}(1/\delta) \log^{4+D} n) \sum_{i=2}^K \Delta_i^{-(4+D)}$. 
The result now follows, as desired.
\end{proof}

\begin{proof} [Proof of Theorem~\ref{theo:eoptimalpac}]
It suffices to choose $n$ large enough such that
\begin{align*}
|\widehat{\text{p-mode}}(S_i(n)) - \theta_i| \le \epsilon/2.
\end{align*}
Indeed, if this were the case, then if arm $i \neq 1$ was selected as the top arm but not $\epsilon$-optimal, then 
\begin{align*}
\theta_i < \theta_1 - \epsilon \Rightarrow \theta_i + \epsilon/2 < \theta_1 - \epsilon/2 \Rightarrow 
\widehat{\text{p-mode}}(S_i(n))  < \widehat{\text{p-mode}}(S_1(n)),
\end{align*}
a contradiction. Now from Theorem~\ref{theo:pmode} with confidence parameter $\delta/K$, it follows that it suffices to take
\begin{align*}
n \ge K (\log(K) + \log(1/\delta))^5 \epsilon^{-5} \log (\epsilon^{-5}),
\end{align*}
as desired.

\end{proof}

\section{Top $m$ arms identification results}

In this section we prove the following Theorem~\ref{theo:topmarm}.

We follow closely the proof template in \citep{jiang2017practical}, changing the argument where necessary. Let $OPT$ denote the optimal set of $m$ arms, and $\bar{OPT}$ its complement.

\begin{lemma}\label{lem:opt_diffs}
For any arm $i \in OPT$ and $j \in \overline{OPT}$, $\theta_i - \theta_j \geq \frac{\widetilde{\Delta}_i + \widetilde{\Delta}_j}{2}$. For any two arms $i,j \in OPT$, $\theta_i - \theta_j = \X_i - \widetilde{\Delta}_j$. 
\end{lemma}

\begin{proof}
For any $i \in OPT$ and $j \in \overline{OPT}$ $\theta_i - \theta_j = \X_i - \X_j - \X_m \geq \widetilde{\Delta}_i + \X_j -\left( \frac{\widetilde{\Delta}_i+\widetilde{\Delta}_j}{2}   \right) = \frac{\widetilde{\Delta}_i + \widetilde{\Delta}_j}{2}$. If $i,j \in OPT$ $\theta_i - \theta_j = (\theta_i - \theta_{m+1}) - (\theta_j - \theta_{m+1}) = \widetilde{\Delta}_i - \widetilde{\Delta}_j$.

\end{proof}

We will use the following bound:

\begin{lemma}\label{lemm:union_bound}
For any arm $i \in [K]$, $\theta_i \in (\hat{\theta}_{i,t} - U(t, \delta') , \hat{\theta}_{i, t} + U(t, \delta'))$ for all $t \geq N_{f_i} $ simultaneously with probability $1-\delta'$.
\end{lemma}

\begin{proof}
Let $c = \sum_{i=N0}^\infty \frac{1}{t^2}$. This lemma is a simple consequence of the union bound since for every $t \geq N_{f_i}$, $\theta_i \in (\hat{\theta}_{i,t} - U(t, \delta') , \hat{\theta}_{i, T_i(t)} + U(t, \delta'))$ with probability $1-\delta'/(c t^2)$, by union bound, this holds for all $t \geq N_{f_i} $ with probability at least $1-\delta$. 
\end{proof}

Now we prove Theorem \ref{theo:toparm}:

\begin{proof}
Define $\mathcal{E}$ the event that for all arms $i \in OPT$ and all $t \geq 1$ it holds that for all arms the true modes lie inside the confidence intervals. In other wors, where for all $t$ such that $\sum_{i=1}^L N_{f_i} \leq t \leq n$, it holds that for all $i \in OPT$, $|\hat{\theta}_{i,T_i(t)} - \theta_i | \leq U(T_i(t), \delta/(2m))$ and for all $j \in \hat{OPT}$, $|\hat{\theta}_{j,T_j(t)} - \theta_j | \leq U(T_j(t), \delta/(2(K-m)))$. By Lemma \ref{lemm:union_bound}The probability of $\mathcal{E}$ is:

\begin{align*}
P(\mathcal{E})  &\geq 1-m(\delta/(2K)) - (K-m)(\delta/(2(K-m)))\\
			&= 1-\delta
\end{align*}

We now proceed to prove that conditioned on $\mathcal{E}$ the algorithm works.

 Suppose for the sake of contradiction the algorithm terminates at time $t$ and returns $H_t \neq OPT$. In this case there exists $i \in H_t \cap \bar{OPT}$ and $j \in L_t \cap OPT$. Recall that $h_t$ is the arm in $H_t$ with the lowest lower confidence bound and $l_t$ is the arm in $L_t$ with the highest upper confidence bound. The definition of $h_t$ and event $\mathcal{E}$ guarantees that conditioning on $\mathcal{E}$,

\begin{equation}
\theta_i > \hat{\theta}_{i,T_i(t)} - U(T_i(t), \delta/(2(K-m))) \geq \hat{\theta}_{h_t, T_{h_t}(t)} - U(T_{h_t}(t), \delta/(2(K-m)))
\end{equation}

Similarly:

\begin{equation}
\theta_j < \hat{\theta}_{j, T_j(t)} + U(T_j(t), \delta/(2m)) \leq \hat{\theta}_{l_t, T_{l_t}(t)} + U(T_{l_t}(t), \delta/(2m))
\end{equation}

The stopping condition at round $t$ implies that:

\begin{equation}
\hat{\theta}_{h_t, T_{h_t}(t)} - U( T_{h_t}(t), \delta/(2(K-m))) \geq \hat{\theta}_{l_t, T_{l_t}(t)} + U(T_{l_t}(t), \delta/(2m))
\end{equation}

The three inequalities above together yield $\theta_i > \theta_j$ whcih contradicts the assumption that $i \in \hat{OPT}$ and $j \in OPT$. Thus our algorithm outputs the correct andwer conditioning on event $\mathcal{E}$.

We upper bound the sample complexity of our algorithm by means of what is known as a charging argument. We define the critical armaat time $t$ denoted by $cr_t$, as the arm that has been pulled fewer times between $h_t$ and $l_t$, in other words, $cr_t = \arg\min_{i \in \{ h_t, l_t\}} T_i(t)$. We charge the critical arm a cost of $1$, no matter whether it s actually pulled at this time step. It remains to upper bound the total cost that we charge each arm. To this end we prove the following two claims:

\begin{enumerate}
\item Once an arm has been sampled a certain number of times it will never be critical in the future
\item The expected number of samples drawn from an arm is lower bounded by the total cost it is charged.
\end{enumerate}

This directly gives an upper bound on the cost that we charge each arm and thus an upper bound on the total sample complexity.

For a fixed time step $t$ define $\hat{\theta}_i = \hat{\theta}_{i,T_i(t)}$ and $r_i = U(T_i(t), \delta/(2K))$. Since $\delta/(2K)$ is smaller than $\delta/(2m)$ and $\delta/(2(K-m))$, $r_i$ is greater than or equal to both $U(T_i(t), \delta/(2m))$ and $U(T_i(t), \delta/(2(K-m)))$. It follows that conditioning on event $\mathcal{E}$, $|\hat{\theta}_i - \theta_i| < r_i$ holds for every arm $i$. 

In the following we show that $r_{cr_t}\geq \Delta_{cr_t}/8$. In other words, let $\rho_i$ denote the smallest integer such that $U(\rho_i, \delta/(2K)) < \Delta_i/8$. Then once arm $i$ has been sampled $\rho_i$ times, it will never become critical later. 
We prove the inequality in the following three cases separately.
\begin{enumerate}
\item [Case 1] $h_t \in \bar{OPT}$, $l_t \in OPT$. Since $h_t \in H_t$ and $l_t \in L_t$, we have $\hat{\theta}_{h_t} \geq \hat{\theta}_{l_t}$. It folows that conditing on $\mathcal{E}$,
\begin{equation*}
\theta_{h_t} + \rho_{h_t} > \hat{\theta}_{h_t} \geq \hat{\theta}_{l_t} > \theta_{l_t} - \rho_{l_t}
\end{equation*}
which implies that
\begin{equation*}
\rho_{h_t} + \rho_{l_t} > \theta_{l_t} - \theta_{h_t} \geq \frac{\widetilde{\Delta}_{h_t} + \widetilde{\Delta}_{l_t}}{2}
\end{equation*}

The last step applies Lemma \ref{lem:opt_diffs}. Recall that arm $cr_t$ has been pulled fewer times than the  other arm up to time $t$, and thus the arm has a larger confidence radius than the other arm. Then,
\begin{equation*}
\rho_{cr_t} > \frac{\rho_{h_t} + \rho_{l_t}}{2} \geq \frac{\widetilde{\Delta}_{h_t} + \widetilde{\Delta}_{l_t}}{4} \geq \frac{\widetilde{\Delta}_{cr_t} }{4}
\end{equation*}

\item[Case 2] $h_t \in OPT$, $l_t \in \bar{OPT}$. Since the stopping condition of our algorithm is not met, we have:

\begin{align*}
\bar{\theta}_{h_t} - \rho_{h_t} &\leq \hat{\theta}_{h_t} - U(T_{h_t}(t), \delta/(2(K-m)))\\
								&< \hat{\theta}_{l_t} + U(T_{l_t}(t), \delta/(2m))\\
								&\leq \hat{\theta}_{l_t} + \rho_{l_t}
\end{align*}
It follows that conditioning on $\mathcal{E}$,
\begin{equation*}
\theta_{h_t} - 2\rho_{h_t} < \hat{\theta}_{h_t} - \rho_{h_t} < \hat{\theta}_{l_t} + \rho_{l_t} < \theta_{l_t} + 2\rho_{l_t}
\end{equation*}
which implies by lemma \ref{lem:opt_diffs}

\begin{equation*}
r_{h_t} + r_{l_t} >\frac{ (\theta_{h_t} - \theta_{l_t} ) }{2} \geq \frac{\widetilde{\Delta}_{h_t} + \widetilde{\Delta}_{l_t}}{4}
\end{equation*}

Therefore:

\begin{equation*}
r_{cr_t} \geq \frac{(r_{h_t} + r_{l_t})}{2} \geq \frac{(\widetilde{\Delta}_{h_t} + \widetilde{\Delta}_{l_t})}{8} \geq \widetilde{\Delta}_{cr_t} /8
\end{equation*}

\item[Case 3] $h_t, l_t \in OPT$ or $h_t, l_t \in \bar{OPT}$. By symmetry, it sufficeds to consider the former case. Since the arm $l_t$, which is among the best $m$ arms is in $L_t$ by mistake, there must be another arm $j$ such that $j \in \bar{OPT} \cap H_t$. Recall that $h_t$ is the arm with the smallest lower confidence bound in $H_t$. Thus we have:

\begin{align}
\theta_{h_t} - 2\rho_{h_t} &< \hat{\theta}_{h_t} - \rho_{h_t} \\
						&\leq \hat{\theta}_{h_t} - U(T_{h_t}(t), \delta/(2(K-m)))\\
						& \leq \hat{\theta}_{j} - U(T_j(t), \delta/(2(K-m))) \\
						&< \theta_j
\end{align}

And it then follows from Lemma \ref{lem:opt_diffs}:

\begin{equation}\label{eq:topm_eq1}
r_{h_t} > (\theta_{h_t} - \theta_j)/2 \geq (\widetilde{\Delta}_{h_t}+\widetilde{\Delta}_j)/4 \geq \widetilde{\Delta}_{h_t}/4
\end{equation}

Thus if $cr_t = h_t$ the claim directly holds. It remains to consider the case $cr_t = l_t$ 
Since $\hat{\theta}_{h_t} \geq \hat{\theta}_{l_t}$ we have:
\begin{equation*}
\theta_{l_t} - \rho_{l_t} < \hat{\theta}_{l_t} \leq \hat{\theta}_{h_t} < \theta_{h_t} + \rho_{h_t}
\end{equation*}

And therefore by Lemma \ref{lem:opt_diffs}:
\begin{equation}\label{eq:topm_eq2}
\rho_{h_t} + \rho_{l_t} > \theta_{l_t} - \theta_{h_t} = \widetilde{\Delta}_{l_t} - \widetilde{\Delta}_{h_t}
\end{equation}

Since we charge $cr_t = l_t$ it holds that $\rho_{l_t} \geq \rho_{h_t}$ and thus by \ref{eq:topm_eq1} and \ref{eq:topm_eq2}.

\begin{equation*}
r_{l_t} \geq (r_{h_t} + r_{l_t})/6 + 2\rho_{h_t}/3 > (\widetilde{\Delta}_{l_t} - \widetilde{\Delta}_{h_t})/6 + \widetilde{\Delta}_{h_t}/6 = \widetilde{\Delta}_{l_t}/6
\end{equation*}
This finishes the proof of the claim.
\end{enumerate}

We note that when we charge arm $cr_t$ with a cost of $1$ at time step $t$, it holds that $T_{cr_t}(t) \leq (T_{h_t}(t) + T_{l_t}(t))/2$. According to the algorithm, arm $cr_t$ is pulled at time $t$ with probability at least $1/2$. Recall that $\rho_i$ is defined as the smallest integer such that $U(\rho_i, \delta/(2K )) < \widetilde{\Delta}_i /8$. Let random variable $X_i$ denote that number of times that arm $i$ is charged before it has been pulled $\rho_i$ times. Since in expectation, an arm will get a sample after being charged at most twice, we have $\mathbb{E}[X_i] \leq 2\rho_i$ 

We have  that 

$$r_{i} \geq \text{PolyLog}\left(\frac{1}{\delta}, \frac{1}{\widetilde{\Delta}_i}\right) \frac{1}{\widetilde{\Delta}_i}$$
implies  $U(t, \delta/2K) \leq \frac{\widetilde{\Delta}_i}{8}$.

Therefore the complexity of the algorithm conditioning on event $\mathcal{E}$ is upper bounded by:

\begin{equation*}
\sum_{i=1}^K \mathbb{E}[X_i] + KN_0 =  \text{PolyLog}\left(\frac{1}{\delta}, \sum_{i=1}^K\frac{1}{\widetilde{\Delta}_i}\right) \sum_{i=1}^K\widetilde{\Delta}_i^{-(4+D)},
\end{equation*}
as desired.
\end{proof}

\section{Proofs of Regret Bounds}

In this section, not included in the main paper we explore the notion of defining regret of an algorithm for $p-$mode identification with respect to the modal values. The loss of pulling an arm at time $t$ is defined as the distance between the $p-$mode of said arm to the $p-$mode of the optimal arm.

We introduce the following notions of regret based on the $p-$modes. 
\begin{align*}
\mathcal{R}(n) &= n \cdot \max_{i=1,...,K} \theta_i - \sum_{j=1}^n \theta_{I_j}, \\
\overline{\mathcal{R}}(n) &= \max_{i=1,...,K} n \cdot \widehat{p-\text{mode}}\left( \{X_{i, t} : 1 \le t \le n\} \right) \\
&- \sum_{i=1}^K T_i(n) \cdot \widehat{p-\text{mode}} (\{X_{i, t} : I_t = i, 1 \le t \le n \}).
\end{align*}
The regret thus rewards the strategy with the $p-$mode ($\mathcal{R}_n$) or the sample $p-$mode ($\overline{\mathcal{R}}(n)$)
 of all trials for a particular arm rather than the mean as in classical formulations.
 
We next give a regret bounds for Algorithm~\ref{alg:ucb}. For $\mathcal{R}(n)$, we attain a poly-logarithmic regret in the number of time steps, while for $\overline{\mathcal{R}}(n)$ we attain a regret of order $\widetilde{O}(n^{4/(4+D)})$. The extra error from the latter is incurred from the errors in the mode estimates.

\begin{proof} [Proof of Theorem~\ref{theo:ucbregret}]

We have by Lemma~\ref{ucbbound} that
\begin{align*}
\mathcal{R}(n) &= \sum_{i = 2}^K T_i(n) \Delta_i \le (2 (\log n + \log(1/\delta)) \log n)^{4+D} \cdot \sum_{i=2}^K \Delta_i^{-(3+D)} + 2 K \cdot N_0 \\
&\le  32 (\log^{8+2D} n + \log^{(4+D)}(1/\delta) \log^{(4+D)} n) \cdot \sum_{i=2}^K \Delta_i^{-(3+D)}  + 2 K \cdot N_0.
\end{align*}

Now, for $T_i(t) > N_0$, repeating what was established in Lemma~\ref{ucbbound}, we have confidence intervals for each arm index $i$ as follows:
\begin{align*}
\theta_i &\in \left( \widehat{\theta}_{i, t} -  \widehat{\sigma}_{i, t},  \widehat{\theta}_{i, t} + \widehat{\sigma}_{i, t},\right),
\end{align*}
where
\begin{align*}
\widehat{\theta}_{i, t} := \widehat{\text{p-mode}}(S_i(t-1)) \text{ and }
\widehat{\sigma}_{i, t} := (\log n + \log(1/\delta)) \cdot \frac{\log(T_i(t-1))}{(T_i(t-1))^{1/(4+D)}}.
\end{align*}
Thus,
\begin{align*}
\overline{\mathcal{R}}(n) &\le n \widehat{\sigma}_{1, n} + \sum_{i = 2}^K T_i(n)\cdot 1\{T_i(n) > N_0 \} \cdot (\widehat{\sigma}_{i, n} + \Delta_i) + \sum_{i = 2}^K T_i(n) \cdot 1\{T_i(n) \le N_0 \}  \\
&= \mathcal{R}(n) + N_0\cdot K + n \widehat{\sigma}_{1, n} + \sum_{i=2}^K T_i(n)\cdot 1 \{ T_i(n) > N_0 \}) \cdot \widehat{\sigma}_{i, n} \\
&\le \mathcal{R}(n) + N_0\cdot K +  n \widehat{\sigma}_{1, n} + (\log n + \log(1/\delta)) \sum_{i=2}^K T_i(n)^{4/(4+D)} \cdot \log(T_i(n)) \\
&\le  \mathcal{R}(n) + N_0\cdot K +  n \widehat{\sigma}_{1, n} + (\log n + \log(1/\delta)) (\log n) \sum_{i=2}^K T_i(n)^{4/(4+D)} \cdot \\
&\le \mathcal{R}(n) + N_0\cdot K +  n \widehat{\sigma}_{1, n} + (\log n + \log(1/\delta)) (\log n) \cdot K^{1/(4+D)} n^{4/(4+D)} 
\end{align*}
Next, we have
\begin{align*}
T_1(n) = n - \sum_{i=2}^K T_i(n) \ge n - 32K(\log^{8+2D} n + \log^{(4+D)}(1/\delta) \log^{(4+D)} n) \cdot \sum_{i=2}^K \Delta_i^{-(3+D)} + 2 K^2 \cdot N_0 \ge n/2,
\end{align*}
where the last inequality holds for $n$ sufficiently large depending on $K,\delta, \Delta_2, \Delta_3, ...,\Delta_K$.
Thus, we have $\widehat{\sigma}_{1, n} \le \frac{1}{2} (\log n + \log(1/\delta)) \log n \cdot n^{1/(4+D)}$. Combining this with the above, we get
\begin{align*}
\overline{\mathcal{R}}(n) &\le \mathcal{R}(n) + O\left( (\log^2 n + \log(1/\delta) \log n  + K ) \cdot n^{4/(4+D)} \right),
\end{align*}
as desired.
\end{proof}

\section{Simulations for $p-$modal bandits}

\begin{figure}[H]
\begin{center}
\includegraphics[width=0.40\textwidth]{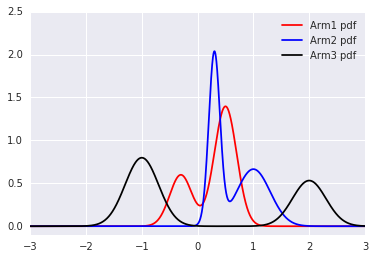}
\includegraphics[width=0.40\textwidth]{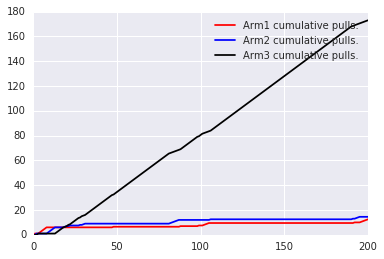}
\end{center}
\label{fig:2}
\caption{Left. Multimodal densities. The score function equals the absolute value of the $2-$mode of the Arm. Under this score function the best arm is Arm 3. Right. Our algorithm quickly figures which one is the best arm and pulls it more than the other ones. }
\end{figure}

\begin{figure}[H]
\begin{center}
\includegraphics[width=0.40\textwidth]{figures/2_mode_density.png}
\includegraphics[width=0.40\textwidth]{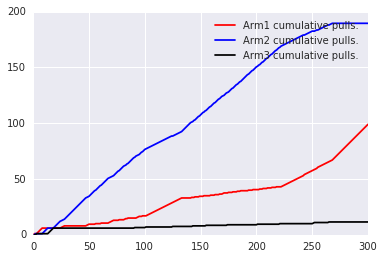}
\end{center}
\label{fig:2}
\caption{Left. Multimodal densities. The score function equals the value of the $2-$mode of the Arm. Under this score function the best arm is Arm 2. Right. Our algorithm quickly figures which one is the best arm and pulls it more than the other ones. }
\end{figure}

\section{Proofs of $p$-mode estimation results for Contextual Bandits}
\begin{assumption} \label{assumption:conditional}
\begin{itemize}
\item Let $0 < \alpha \le 1$. For all $i \in [K]$, $x, x' \in \mathcal{X}$, $r, r' \in [0, 1]$, the following holds.
\begin{align*}
|f_i (r, x) - f_i(r', x')| \le C_\alpha |(r,x)-(r',x')|^\alpha,
\end{align*}
where $(r, x)$ represents the concatenation of $r$ and $x$ into a vector in $\mathbb{R}^{d+1}$.
\item
The local maxima of $f_i(r | X = x)$ are points for all $i$ and $x \in \mathcal{X}$.
\item
There exists $\check{C}, \hat{C}, r_M > 0$ such that the following holds simulatenously for all $i$, $r$ and $x \in \mathcal{X}$.
\begin{align*}
\check{C} (r_0 - r)^2 \le |f_i(r_0, x) - f_i(r, x)| \le \hat{C}(r_0 - r)^2,
\end{align*}
for $r \in B(r_0, r_M)$ where $r_0$ is a mode of $f_i(\cdot | X = x)$.
\item $\inf_{x \in \mathcal{X}} \inf_{r \in \{\text{modes of } f(r|X=x)\} +[-r_M, r_M]} f(r, x) > 0.$
\end{itemize}
\end{assumption}
In this section we provide guarantees for Algorithm \ref{alg:pmode_conditional_general}, a generalization of Algorithm \ref{alg:pmode_conditional}.

\begin{algorithm}[H]
   \caption{Estimating $p$-modes of conditional densities}
   \label{alg:pmode_conditional_general}
\begin{algorithmic}
	\STATE Input: $m$, $k$, samples $S_{[n]} := \{(r_1, X_1),...,(r_n,X_n)\}$, $x \in \mathbb{R}^d$, confidence parameter $\delta$, and $\epsilon > 0$.
	\STATE Let $r_k(r,x)$ be the distance from $(r, x) \in \mathbb{R} \times \mathbb{R}^{d}$ to its $k$-th nearest neighbor in $S_{[n]}$. 
   \STATE $\widehat{f}(r, x) := k/(v_{d+1} \cdot n\cdot r_k(r,x)^{d+1})$, where $v_{d+1}$ is the volume of $(d+1)$-dimensional unit ball.
   \STATE
   \STATE Initialize $\mathcal{A}:= \emptyset$. Define $\beta_k =  100\cdot \log(1/\delta) \cdot \sqrt{\log n} / \sqrt{k}$. 
   \STATE Let $R := \{ \frac{1}{m}, \frac{2}{m}, ..., \frac{m-1}{m} \}$. Sort $R$ in descending order by $\widehat{f}(r, x)$ value for $r \in R$ and let $R_i$ denote the $i$-th element of $R$.
   \STATE Define $G(\lambda)$ to be the graph with vertices $r \in R$ such that $\widehat{f}(r, x) > \lambda$ and edges between $r_1$ and $r_2$ iff $|r_1 - r_2| \le \min\{r_k(r_1, x), r_k(r_2, x) \}$.
   \STATE
   \FOR{$i=1$ {\bfseries to} $m$}
   \STATE Define $\lambda := \widehat{f}(R_i, x)$ and let $A$ be the CC of $G(\lambda - \beta_k \lambda - \epsilon)$ that contains $R_i$.
   \STATE {\bf if} $A$ is disjoint from all sets in $\mathcal{A}$ {\bf then}  Add $A$ to $\mathcal{A}$. {\bf end if}
   \STATE {\bf if} $|\mathcal{A}| = p$ {\bf then}  break. {\bf end if}
   \ENDFOR
   \STATE \textbf{return} $\min \{ \text{argmax}_{x \in A} f_k(x) : A \in \mathcal{A} \}$. 
\end{algorithmic}
\end{algorithm}

We utilize uniform $k$-NN bounds from \citep{dasgupta2014optimal}, which are repeated here. For these results $f$ is an arbitrary density on $\mathbb{R}^d$ and $f_k$ is its $k$-NN density estimator from a finite sample of size $n$ defined as $f_k(x) := k/(n\cdot v_d \cdot r_k(x)^d)$ where $v_d$ is the volume of a unit ball in $\mathbb{R}^d$.

\begin{definition}
\begin{align*}
\hat{r}(\epsilon, x) &:=\sup\left\{r : \sup_{x' \in B(x, r) \cap M} f(x') - f(x) \le \epsilon \right\}\\
\check{r}(\epsilon, x) &:=\sup\left\{r : \sup_{x' \in B(x, r) \cap M} f(x) - f(x') \le \epsilon \right\}.
\end{align*}
\end{definition}
\begin{definition}
$C_{\delta, n} := 16 \log(2/\delta) \sqrt{d \log n}$.
\end{definition}

\begin{lemma}\label{upperboundlemmaknn} [Lemma 3 of \citep{dasgupta2014optimal}]
Suppose that $k \ge 4 C_{\delta, n}^2$. Then with probability at least $1-\delta$, the following holds for all $x\in \mathbb{R}^d$ and $\epsilon > 0$.
\begin{align*}
f_k(x) < \left(1 + 2\frac{C_{\delta, n}}{\sqrt{k}} \right) (f(x) + \epsilon),
\end{align*}
provided $k$ satisfies $v_d \cdot \hat{r}(x, \epsilon) \cdot (f(x) + \epsilon) \ge \frac{k}{n} + C_{\delta, n}\frac{\sqrt{k}}{n}$.
\end{lemma}
\begin{lemma}\label{lowerboundlemmaknn} [Lemma 4 of \citep{dasgupta2014optimal}]
Suppose that $k \ge 4 C_{\delta, n}^2$. Then with probability at least $1-\delta$, the following holds for all $x\in \mathbb{R}^d$ and $\epsilon > 0$.
\begin{align*}
f_k(x) \ge \left(1 -\frac{C_{\delta, n}}{\sqrt{k}} \right) (f(x) - \epsilon),
\end{align*}
provided $k$ satisfies $v_d \cdot \check{r}(x, \epsilon) \cdot (f(x) - \epsilon) \ge \frac{k}{n} - C_{\delta, n}\frac{\sqrt{k}}{n}$.
\end{lemma}

If $f$ is $\alpha$-H\"older continuous for some $\alpha \in (0, 1]$ (i.e. $|f(x) - f(x')| \le C_\alpha |x - x'|^\alpha$), then
we can make the observation that $\max\{ \check{r}(\epsilon, x), \hat{r}(\epsilon, x) \} \le (\epsilon / C_\alpha)^{1/\alpha}$.
Then applying this to Lemma~\ref{upperboundlemmaknn} and~\ref{lowerboundlemmaknn}, we have the following.
\begin{corollary} \label{holderrates}[Finite sample rates for H\"older densities]
Let $\delta > 0$.
Suppose that $f$ is $\alpha$-H\"older continuous for some $\alpha \in (0, 1]$. 
and let $f$ have compact support $\mathcal{X} \subseteq \mathbb{R}^d$.
Suppose that $\inf_{x \in \mathcal{X}} f(x) \ge \lambda_0$. 
Then exists constant $C$ depending on $f$ such that the following holds if $n \ge C_{\delta, n}^2$ with probability at least $1 - \delta$.
\begin{align*}
\sup_{x \in \mathcal{X}} |\widehat{f}(x) - f(x)| \le C\left( \frac{C_{\delta, n}}{\sqrt{k}} + \left( \frac{k}{n}\right)^{\alpha/d} \right).
\end{align*}
\end{corollary}

\begin{theorem}\label{constrained_pmode}[Estimating constrained modes]
Suppose we have $n$ i.i.d. samples $S_{[n]}$ from $f$. Let $C_0 > 0$.
There exists $N_f$ depending on $C_0$, $r_M$, $||f||_\infty$, $\inf_{x \in \mathcal{X}} \inf_{r \in \{\text{modes of } f(r|X=x)\} + [-r_M, r_M]} f(r, x)$, $K$, $\hat{C}$, $\check{C}$, $C_\alpha$, $\alpha, \delta$ such that for $n \ge N_f$, $k = n^{2\alpha/(2\alpha+d+1)}$, $m = 1/k^2$, $\epsilon = \Omega(1/\log n)$ and $\epsilon\rightarrow 0$ we have the following.
\begin{align*}
\mathbb{P}\Bigg(\sup_{x \in \mathcal{X}} |\widehat{\text{mode}}(S_{[n]}, x) - \text{mode}(f(r|X = x))| \ge C_0 \frac{\sqrt{\log(1/\delta)} \log n}{n^{\alpha/(4\alpha+2d+2)}}\Bigg) \le \delta.
\end{align*}
where $\widehat{\text{mode}}(S_{[n]}, x)$ is the output of Algorithm~\ref{alg:pmode_conditional} for context $x \in \mathcal{X}$.
\end{theorem}

\begin{proof}[Proof of Theorem~\ref{constrained_pmode}]
The proof will proceed as follows. First (step 1), we show a bound on recovering a mode when constrained to a region of the reward space where the conditional density has only a single mode. Next (step 2), we show that each mode has a corresponding estimate. Then (step 3), we show that the Algorithm does not choose any false modes. Finally (step 4), we show that the first $p$ recovered modes indeed will give us the $p$-mode.

\underline{Step 1: Single-mode recovery}.
Suppose that closed interval $I \subseteq [0, 1]$ is such that $f(r|X = x)$ has only one mode $r_0$ in $I$.
Then take $\tilde{r} = 2\sqrt{2C\left( \frac{C_{\delta, n}}{\sqrt{k}} + (k/n)^{\alpha/(d+1)}\right)}$. Then it suffices to show
\begin{align*}
\inf_{r \in [r_0- 1/m, r_0 + 1/m]} \widehat{f}(r, x) > \sup_{r \in A \backslash [r_0-\tilde{r}, r_0+\tilde{r}]} \widehat{f}(r, x).
\end{align*}
We have 
\begin{align*}
\sup_{r \in I \backslash [r_0 - \tilde{r} / 2, r_0 + \tilde{r}/2]} f(r, x) \le f(r_0, x) - \check{C}(\tilde{r} / 2)^2.
\end{align*}
Then, we have by Corollary~\ref{holderrates}:
\begin{align*}
\sup_{r \in I \backslash [r_0-\tilde{r}, r_0+\tilde{r}]} \widehat{f}(r, x) < f(r_0, x) - \check{C}(\tilde{r} / 2)^2 + C\left( \frac{C_{\delta, n}}{\sqrt{k}} + \left( \frac{k}{n}\right)^{\alpha/(d+1)} \right).
\end{align*}
On the other hand, we have
\begin{align*}
\inf_{r \in [r_0- 1/m, r_0+1/m]} f(r, x) \ge f(r_0, x) - \hat{C}(1/m)^2.
\end{align*}
Then, by Corollary~\ref{holderrates}:
\begin{align*}
\inf_{r \in [r_0- 1/m, r_0+1/m]} \widehat{f}(r, x) \ge f(r_0, x) - \hat{C}(1/m)^2 - C\left( \frac{C_{\delta, n}}{\sqrt{k}} + \left( \frac{k}{n}\right)^(\alpha/d) \right).
\end{align*}
It suffices to have
\begin{align*}
 f(r_0, x) - \hat{C}(1/m)^2 - C\left( \frac{C_{\delta, n}}{\sqrt{k}} + \left( \frac{k}{n}\right)^{\alpha/(d+1)} \right) >  
 f(r_0, x) - \check{C}(\tilde{r} / 2)^2 + C\left( \frac{C_{\delta, n}}{\sqrt{k}} + \left( \frac{k}{n}\right)^{\alpha/(d+1)} \right).
\end{align*}
This holds for $n$ sufficiently large for all $i \in [K]$.
Thus, it follows that $\inf_{r \in [r_0-1/m, r_0+1/m]} \widehat{f}(r, x) > \sup_{r \in I \backslash [r_0- \tilde{r}, r_0+\tilde{r}]} \widehat{f}(r, x).
$ and so $|\hat{r} - r_0| \le \tilde{r}$.\\ 

\underline{Step 2: Every mode is estimated.} This part of the proof mirrors that of Lemma 7 and 8 of \citep{jiang2017modal}.
We show that for every $r_0$ that is a mode of $f(r | X= x)$ has a corresponding estimate in Algorithm~\ref{alg:pmode_conditional}.
First, we show that $B(r_0, \tilde{r})$ and $B(r_0, r_M) \backslash B(r_0, \bar{r})$ are disconnected in $G(\lambda - \beta_k \lambda-\epsilon)$ where $\bar{r} = r_M / 3$ and $\lambda = \max_{r \in R \cap B(r_0, r_M)} \widehat{f}(r, x)$. Let $\bar{F} = f(r_0, x) - \check{C}(\bar{r}/2)^2$. Then for all $r \in B(r_0, r_M) \backslash B(r_0, \bar{r})$, we have $\hat{r}(\bar{F} - f(r, x), x) > \bar{r}/2$. Thus applying Lemma~\ref{upperboundlemmaknn}, we have
\begin{align*}
\sup_{r \in B(r_0, r_M) \backslash B(r_0, \bar{r})} \widehat{f}(r, x) < \left(1 + 2\frac{C_{\delta, n}}{\sqrt{k}}\right) \bar{F} \le \lambda - \beta_k \lambda-\epsilon.\
\end{align*}
where the last inequality holds for $n$ sufficiently large.
This shows that $G(\lambda - \beta_k \lambda)$ contains no vertex in $B[r_0- r_M, r_0+r_M] \backslash [r_0- \bar{r},r_0+\bar{r}]$.
Next, define $S := \{r : \bar{r} \le |r - r_0| \le 2\bar{r} \}$. We can similarly apply Lemma~\ref{upperboundlemmaknn} to show
\begin{align*}
\sup_{r \in S} \widehat{f}(r, x) < \lambda - \beta_k \lambda - \epsilon.
\end{align*}
Now for $n$ sufficiently large, $\tilde{r} < \bar{r}$ and thus $[r_0-\tilde{r},r_0+\tilde{r}]$ and $B[r_0- r_M, r_0+r_M] \backslash [r_0- \bar{r},r_0+\bar{r}]$ are disconnected in $G(\lambda - \beta_k \lambda)$.

Thus, the algorithm will choose $\argmax_{r \in R \cap B(r_0, r_M)} \widehat{f}(r, x)$ as an estimate of $r_0$.

\underline{Step 3: No false modes are estimated.}
It suffices to show that if $A_1$ and $A_2$ are sets of points with empirical density at least $\lambda$ in separate connected components of $G(\lambda - \beta_k \lambda - \epsilon)$, then the following holds. If $\lambda_f := \inf_{r \in A_1 \cup A_2} f(r, x)$, then $A_1$ and $A_2$ are disconnected in $\{ r \in [0, 1] : f(r, x) > \lambda_f \}$. This follows from standard results in cluster tree estimation. e.g. Lemma 10 of \citep{jiang2017modal} or Lemma 6 of \citep{dasgupta2014optimal}.

\underline{Step 4: $p$-mode identification.} 
If $\mathcal{M} = \{r_1,...,r_\ell\}$ are the modes where $f(r_1, x) \ge f(r_2, x) \ge \cdots \ge f(r_p, x) > f(r_{p+1}, x) \ge \cdots \ge f(r_\ell, x)$, it thus suffices to show that
\begin{align*}
\min_{j=1,...,p} f_k(\widehat{r}_j) > \max_{j=p+1,...,|\mathcal{M}|} f_k(\widehat{r}_j).
\end{align*}
where $\widehat{r}_j$ is the estimate corresponding to $r_j$.
We have
\begin{align*}
\min_{j=1,...,p} f(r_j, x) - \max_{j=p+1,...,|\mathcal{M}|} f(r_j,x) = f(r_p,x) - f(r_{p+1},x):= \delta_p > 0.
\end{align*}
Earlier it was shown that $\widehat{r}_j \in B(r_j, \tilde{r})$.
It thus suffices to show that
\begin{align*}
\min_{j=1,...,p} \inf_{x \in B(r_j, \tilde{r})} \widehat{f}(r, x)  > \max_{j=p+1,...,|\mathcal{M}|} \sup_{x \in B(r_j, \tilde{r})} \widehat{f}(r, x).
\end{align*}
We have
\begin{align*}
\inf_{x \in B(r_j, \tilde{r})} f(x) \ge  f(r_j) - \hat{C} \tilde{r}^2.
\end{align*}
By Lemma 4 of \citep{dasgupta2014optimal}
\begin{align*}
\inf_{r \in B(r_j, \tilde{r})} \widehat{f}(r, x) \ge  \left(1 - \frac{C_{\delta, n}}{\sqrt{k}}\right)\cdot (f(r_j, x) - \hat{C} \tilde{r}^2).
\end{align*}
By Lemma 3 of \citep{dasgupta2014optimal},
\begin{align*}
\sup_{r \in B(r_j, \tilde{r})} \widehat{f}(r, x) \le \left(1 + \frac{C_{\delta, n}}{\sqrt{k}}\right)\cdot f(r_j, x).
\end{align*}
Combining these two, for $n$ sufficiently large,
\begin{align*}
\left(1 - \frac{C_{\delta, n}}{\sqrt{k}}\right) \cdot(f(r_p, x) - \hat{C} \tilde{r}^2) \ge \left(1 + \frac{C_{\delta, n}}{\sqrt{k}}\right)\cdot f(r_{p+1}, x).
\end{align*}
The result follows immediately.

\end{proof}

\section{Proof of Contextual Bandit Results}
Theorem~\ref{theo:contextual} follows from the result below.
\begin{theorem} \label{theo:contextual_aux}
Let $P(\Delta)$ represent the region of $\mathcal{X}$ where the gap between the $p$-mode of the best arm and the second best arm is at least $\Delta$. Formally:
\begin{align*}
P(\Delta) := \{ x\in \mathcal{X} : \{ \text{p-mode}(f_i(r|X=x)) : i \in [k]  \}_{(1)} -  \{ \text{p-mode}(f_i(r|X=x)) : i \in [k]  \}_{(2)} > \Delta \},
\end{align*} where $S_{(i)}$ denotes the $i$-th largest element of a set with ties broken arbitrarily.
Let $\mathcal{M}_{i,x}(p)$ denote the $p$-th highest density mode of $f_i(r|X =x)$ where ties are broken arbitrarily. 
Let $Q(\Delta)$ represent the region of $\mathcal{X}$ where
the $p$-mode is salient enough to be detected via a density difference of $\Delta$. Formally:
\begin{align*}
Q(\Delta) := \{ x \in \mathcal{X} : \min_{i \in [K]} f_i(\mathcal{M}_{i, x}(p)) - f_i(\mathcal{M}_{i, x}(p+1)) > \Delta\}.
\end{align*}
If
\begin{align*}
n > \max \left\{K\cdot(\log(1/\delta) + \log(K))^{(4\alpha + 2d + 2)/\alpha}  (\Delta/3)^{-(4\alpha + 2d + 2)/\alpha} \log((\Delta/3)^{-(4\alpha + 2d + 2)/\alpha}), K \cdot \max_{i \in [K]} N_{f_i} \right\},
\end{align*}
then Algorithm~\ref{alg:uniformcontextual} with confidence parameter $\delta/K$ satsifies
\begin{align*}
\mathbb{P}\Bigg(\widehat{\pi}(x) = \pi(x) \hspace{0.2cm} \forall x \in P(\Delta) \cap Q(\Delta)\Bigg) > 1 - \delta.
\end{align*}
In particular, as $\Delta \rightarrow 0$, $\widehat{\pi}$ gives the correct policy {\it uniformly} over $\mathcal{X}$ wherever the $p$-mode is well-defined and there exists a unique optimal arm.
\end{theorem}
\begin{proof} [Proof of Theorem~\ref{theo:contextual_aux}]
Let $n_1 = n / K$.
In light of Theorem~\ref{constrained_pmode}, it suffices to have $n_1 \ge \max_{i\in[K]} N_{f_i}$ and
\begin{align*}
\frac{(\log(1/\delta) + \log(K)) \log n_1}{n_1^{\alpha/(4\alpha + 2d + 2)}} \le \Delta/3.
\end{align*}
The latter is equivalent to
\begin{align*}
\frac{n_1}{(\log(1/\delta) + \log(K))^{(4\alpha + 2d + 2)/\alpha}} (\log n_1)^{(4\alpha + 2d + 2)/\alpha}\ge (\Delta/3)^{-(4\alpha + 2d + 2)/\alpha}.
\end{align*}
It thus suffices to have
\begin{align*}
n_1 \ge (\log(1/\delta) + \log(K))^{(4\alpha + 2d + 2)/\alpha} (\Delta/3)^{-(4\alpha + 2d + 2)/\alpha} \log((\Delta/3)^{-(4\alpha + 2d + 2)/\alpha}),
\end{align*}
as desired.
\end{proof}

\section{Infinite Armed Bandit Result and proof}
Here, we consider the setting where the arms $\mathcal{A}$ is a compact subset of $\mathbb{R}^d$. 
Suppose that the reward density function of arm $a \in \mathcal{A}$ at reward $r \in [0, 1]$ is $f(r, a)$ where $f : [0,1] \times \mathcal{A} \rightarrow \mathbb{R}$.
The goal is to find the top-arm defined by
\begin{align*}
\argmax_{a \in \mathcal{A}} \text{mode}(f(\cdot, a)).
\end{align*}
Our algorithm starts by initializing a ball of candidate arms, which contains the entire action space. 
Then, it repeats the following steps: (1) sample $M$ arms from the candidate region, (2) run the UCB strategy on these $M$ arms, (3) update the candidate region to the ball whose center is the best arm from the last step and half of the radius as before. We show that if $M$ and $P$ are chosen sufficiently large, the candidate region will always contain an optimal arm, until a certain point. After that point, we will remain within the desired error from the the optimal arm. 
\begin{algorithm}[tbh]
   \caption{Zooming UCB Strategy}
   \label{alg:zoomingucb}
\begin{algorithmic}
	\STATE Input: Total time $n$ and confidence parameter $\delta$. $M$, number of arms to maintain in the active set, $P$, the length of time to maintain an active set before updating it. $R_0$, starting radius and $A_0 \in \mathcal{A}$.
	\STATE Define $S_a(t)$ be the rewards observed from arm $a \in \mathcal{A}$ up to and include time $t$.
	\STATE Let $T_a(t)$ be the number of times arm $a$ was pulled up to and including time $t$. i.e. $|S_a(t)| = T_a(t)$.
	\STATE Let $\widetilde{\mathcal{A}}$ be a sample $M$ arms uniformly on $B(A_0, R_0) \cap \mathcal{A}$. This will be the set of active arms. 
	\FOR{$t = 1,...,n$}
		\IF{$t \text{ modulo } P \equiv 0$}
			\STATE Update 
			\begin{align*}
			A_0 = \argmax_{a \in \widetilde{\mathcal{A}}} \Bigg\{ \widehat{\text{mode}}(S_a(t-1)) 
			+ \frac{\log(1/\delta) \cdot \log(T_a(t-1))}{(T_a(t-1))^{1/5}} \Bigg\}.
			\end{align*}
			\STATE $R_0 = R_0 / 2$.
			\STATE Update $\widetilde{\mathcal{A}}$ to a sample $M$ arms uniformly on $B(A_0, R_0) \cap \mathcal{A}$.
		\ENDIF
		\STATE Pull arm
		\begin{align*}
			I_t = \argmax_{a \in \widetilde{\mathcal{A}}} \Bigg\{ \widehat{\text{mode}}(S_a(t-1)) 
			+ \frac{\log(1/\delta) \cdot \log(T_a(t-1))}{(T_a(t-1))^{1/5}} \Bigg\}.
			\end{align*}
	\ENDFOR
\end{algorithmic}
\end{algorithm}
We make the following regularity assumptions on $f$. The first is that $f$ is H\"older continuous over the reward and arms.
\begin{assumption}
\begin{itemize}
\item Let $0 < \alpha \le 1$. For all $a, a' \in \mathcal{A}$, $r, r' \in [0, 1]$, the following holds.
\begin{align*}
|f(r, a) - f(r', a')| \le C_\alpha |(r,a)-(r',a')|^\alpha,
\end{align*}
where $(r, a)$ represents the concatenation of $r$ and $a$ into a vector in $\mathbb{R}^{d+1}$.
\item The local maxima of $f(r, a)$ over $r \in [0, 1]$ with $a \in \mathcal{A}$ fixed are singleton points for all $a \in \mathcal{A}$.
\item  There exists $\check{C}, \hat{C}, r_M > 0$ such that the following holds simulatenously for all $a \in \mathcal{A}$, $r$ and $x \in \mathcal{X}$.
\begin{align*}
\check{C} (r_0 - r)^2 \le |f(r_0, a) - f(r, a)| \le \hat{C}(r_0 - r)^2,
\end{align*}
for $r \in B(r_0, r_M)$ where $r_0$ is a mode of $f(\cdot, a)$.
\item $\inf_{a \in \mathcal{A}} \inf_{r \in \{\text{maximas of } f(\cdot, a)\} +[-r_M, r_M]} f(r, a) > 0.$
\end{itemize}
\end{assumption}
We now give the result for Algorithm~\ref{alg:zoomingucb}. It states that given enough samples, and parameter chosen appropriately,
the algorithm will recover an approximately optimal arm with high probability.
\begin{theorem} \label{theo:infinite}
Let $\Delta, \delta > 0$.
If $n \ge  \widetilde{\Omega}\left(\Delta^{-5 - d/\alpha} \right)$ then Algorithm~\ref{alg:zoomingucb} with confidence parameter $\delta$, $M = \widetilde{\Omega}\left(\Delta^{-d/\alpha} \right)$,
$P = \widetilde{\Omega}\left(\Delta^{-5}\right)$,
 $R_0 = \text{diam}(\mathcal{A})$ and arbitrary $A_0 \in \mathcal{A}$ satisfies: $\mathbb{P}\Bigg(I_n \text{ is $\Delta$-optimal}  \Bigg) > 1 - \delta$. 

In particular, as $\Delta \rightarrow 0$, $\widehat{\pi}$ gives a $\Delta$-optimal policy {\it uniformly} over $\mathcal{X}$.
\end{theorem}

Suppose that $a^* \in \argmax_{a \in \mathcal{A}} \text{p-mode}(f(\cdot, a))$.
Then for $r_0 = (\Delta/(2C_\alpha))^{1/\alpha}$, we have that for all $a \in B(a^{*}, r) \cap \mathcal{A}$, 
we have 
\begin{align*}
\text{p-mode}(f(\cdot, a^*)) - \text{p-mode}(f(\cdot, a)) < \Delta/3.
\end{align*}
However, for $M \ge \frac{C_\delta, n}{v_d} \left(\frac{\Delta}{3 C_\alpha} \right)^{-d/\alpha}$, then by Lemma 7 of \citep{chaudhuri2010rates} we have with probability $1 - \delta$ that
sampling $M$ arms from $\mathcal{A}$ gives us at least one arm in $a \in B(a^{*}, r) \cap \mathcal{A}$.

The result then follows from using our UCB results for the $M$ sampled arms.

\end{document}